\documentclass{article}
\usepackage[a4paper]{geometry}
\usepackage{xcolor}
\usepackage{xspace}
\usepackage{tikz}
\usepackage{amsmath,amsthm,amsfonts,mathtools}
\newtheorem{lemma}{Lemma}
\newtheorem{definition}{Definition}
\newtheorem{theorem}{Theorem}
\usetikzlibrary{calc}
\usetikzlibrary{decorations.pathmorphing}
\usetikzlibrary{positioning}
\usetikzlibrary{backgrounds}
\usepackage[algo2e,ruled,vlined]{algorithm2e}
\SetAlgoSkip{}
\allowdisplaybreaks[4]

\DeclareMathOperator{\vol}{vol}
\DeclareMathOperator{\harm}{Harm}
\DeclareMathOperator{\prob}{Prob}
\newenvironment{proofof}[1]{\begin{proof}[Proof of #1]}{\end{proof}}
\newcommand{\oona}{(1+1)~NA\xspace}
\newcommand{\ooea}{(1+1)~EA\xspace}
\newcommand{\oneoneea}{\ooea}
\newcommand{\R}{\mathbb{R}}
\newcommand{\N}{\mathbb{N}}
\newcommand{\ltwonorm}[1]{\lVert #1\rVert_2}
\newcommand{\abs}[1]{\lvert #1\rvert}
\newcommand{\card}[1]{\abs{#1}}
\newcommand{\expect}[1]{\mathrm{E}[#1]}
\newcommand{\Var}[1]{\mathrm{Var}[#1]}
\renewcommand{\phi}{\varphi}
\newcommand{\dphi}[1]{d_{\phi}^{(#1)}}
\newcommand{\db}[1]{d_b^{(#1)}}
\newcommand{\ie}{i.\,e.\xspace}
\newcommand{\wrt}{w.\,r.\,t.\xspace}
\newcommand{\eg}{e.\,g.\xspace}
\newcommand{\halfsphere}{\textsc{Half}\xspace}
\newcommand{\quartersphere}
{\textsc{Quarter}\xspace}
\newcommand{\twospherequarters}
{\textsc{TwoQuarters}\xspace}
\newcommand{\spherelocalopt}
{\textsc{LocalOpt}\xspace}
\newcommand{\onemax}{\textsc{OneMax}\xspace}
\title{First Steps Towards a Runtime 
Analysis of Neuroevolution}
 \author{Paul Fischer\\DTU Compute\\Technical University of Denmark\\
  Kgs. Lynbgy\\
  Denmark
  \and
 Emil Lundt Larsen$^1$\\DTU Compute\\Technical University of Denmark\\
  Kgs. Lynbgy\\
  Denmark
\and
Carsten Witt
\\DTU Compute\\Technical University of Denmark\\
  Kgs. Lynbgy\\
  Denmark}
\begin{document}

\maketitle

\begin{abstract}
We consider a simple setting in neuroevolution 
where an evolutionary algorithm optimizes 
the weights and activation functions of 
a simple artificial neural network. We 
then define simple example functions 
to be learned by the network and conduct rigorous runtime analyses for networks 
with a single neuron and for a more advanced 
structure with several neurons and 
two layers. Our results show that 
the proposed algorithm is generally efficient 
on two example problems designed for
one neuron and efficient 
with at least constant probability on the 
example problem for a two-layer network. 
In particular, the so-called harmonic mutation operator choosing steps 
of size~$j$ with probability proportional to~$1/j$ turns out as a good choice for the underlying search space.
However, for the case of one neuron, we 
also identify situations with 
hard-to-overcome local optima. Experimental
investigations of our neuroevolutionary 
algorithm and a state-of-the-art CMA-ES 
support the theoretical findings.
\end{abstract}
\footnotetext[1]{The author now works at Abzu ApS, Copenhagen.}

\section{Introduction}
The term \emph{neuroevolution}
describes the generation  
and iterative improvement
of artificial neural networks (ANNs) by 
means of evolutionary computation. Neuroevolution 
is applied in scenarios where classical techniques like backpropagation for the optimization 
of network weights are not available 
or not satisfactory. Moreover, neuroevolution 
allows for the automated
optimization of network 
topology, \ie, the number of neurons, layers, 
and their interconnecting structure in the network, 
which may be a time-consuming manual task 
otherwise. 
Neuroevolution dates back to the late 1980s \cite{MontanaD89,WhitleyHansonICGA89} 
but has become increasingly popular in recent
years along with several breakthroughs in 
the field of artificial intelligence, most notably, 
so-called deep neural networks. For a broader overview, we 
refer the reader to the recent surveys \cite{GalvanMooneyNeuroevolutionSurvey, StanleyCLM19,UnalEvolutionaryDesignNN} 
 on 
neuroevolution and the strongly related field of evolutionary
neural architecture search (which focuses on 
the optimization of neural network topology rather than the weights belonging 
to neurons). 

Evolutionary algorithms (EAs) 
are nature-inspired, heuristic  
optimization techniques applied in 
virtually all engineering disciplines. There 
is huge empirical knowledge on their 
application, but also an increasingly 
solid theory that guides the design 
and application of EAs. In particular, 
theoretical runtime analysis has become 
an established branch in the theory of evolutionary 
computation that enables such results; see the works  
\cite{NeumannW10,Jansen13,DoerrNeumann20,DNTELOSurvey21} for an overview 
of classical and recent results. The first results 
from the 1990s considered extremely simplified 
EAs like the famous \oneoneea on the simple 
\onemax benchmark function. Such initial analyses have paved 
the way toward the analysis of more realistic, population-based 
EAs on advanced benchmarks and classical combinatorial optimization problems. Moreover, 
runtime analysis has led to theoretically grounded advice on parameter choices 
in EAs and the development of new, high-performing variants of~EAs. 

Despite these advancements in the theory of EAs and the huge 
empirical 
success of neuroevolution, we are not aware of any 
theoretical runtime analyses of neuroevolution. The 
aim of this paper is to be a starting point 
for such an analysis. We will suggest a simple
optimization environment 
in neuroevolution inspired 
by the simple evolutionary algorithms 
mentioned above (\eg, the \ooea) that evolves 
the parameters of neurons  
and suggest optimization problems 
dealing with the classification of 
certain points on the unit hypersphere. 
By giving the two halves of the hypersphere 
opposite labels and discretizing the setting, we arrive 
at a simple problem 
that could take the role of
a kind of \emph{OneMax of Neuroevolution}.

The  
first environment we investigate is restricted to the simplest 
possible network of one neuron with binary 
activation function only, where the evolutionary 
algorithm evolves the bias of the activation function and the
weights of the inputs in a representation as a polar angle. We find that the algorithm 
is generally efficient on problems in 
two dimensions where an arc of constant size 
of the unit circle has to be classified positively. 
Afterwards, we  will 
extend the environment to an arbitrary number of neurons 
and two layers and present modifications 
of the classification problems on the unit hypersphere that 
require more than one neuron to be solved exactly. Moreover, we will 
present problems with local optima 
which are hard to overcome. While 
analyzing the runtimes, we compare different 
mutation operators, more precisely a local 
one and the harmonic mutation operator introduced 
in \cite{DietzfelbingerRWWAlgorithmica11}, and prove 
 exponentially (in the desired resolution of the 
 discretized search space) 
 smaller bounds for the harmonic mutation in 
 several cases. 
  Our proposed classification 
problems may serve as examples of typical 
optimization (sub)scenarios in neuroevolution and 
as a starting point for the runtime analysis of more 
advanced scenarios. 

This paper is structured as follows. In Section~\ref{sec:preliminaries}, we introduce the 
formal background on neural networks and their 
parametrization as well as the 
proposed neuroevolutionary algorithm and  
 benchmark problems. Section~\ref{sec:runtime-analysis} 
 proves the concrete runtime results. Section~\ref{sec:experiments} is devoted to experimental 
 supplements, before we finish with some conclusions.
  Detailed experimental data can be found 
  in the appendix.

\section{Preliminaries}
\label{sec:preliminaries}

In this section, we present the foundations of ANNs for classification problems that are relevant for our study, 
define a simple evolutionary algorithm for neuroevolution, and example problems that will be used 
in our theoretical and empirical studies. For a broader introduction 
to ANNs and machine learning, see, \eg, \cite{SuttonBartoRL2018}.

\subsection{Artificial Neuron}
\label{sec:oneneuron}

We are considering artificial neurons with $D$ inputs and a binary threshold activation
function, \ie, the output is $0$ or~$1$. Such a neuron is sometimes called 
\emph{perceptron.} It has $D+1$ parameters,
the input weights $w_1,\dots,w_D$ and the thresholding value $t$.
Let $x=(x_1,\dots,x_D) \in\R^D$ be the inputs of the neuron. The neuron outputs $1$ if
\begin{equation}
w_1x_1+w_2x_2+\dots+w_Dx_D \ge t \label{neuro1:eqn}
\end{equation}
and $0$ otherwise. See Figure~\ref{fig:perceptron} for an 
illustration. The equation can be normalized such that $t=1$. 

A single neuron can be considered as a minimal,
one-layer neural ``network''. 

\begin{figure}[ht]
\centerline{\begin{tikzpicture}[scale=0.9]
\node (Input-1) at (0,-1) {$x_1$};
\node (Input-2) at (0,-1.6) {$x_2$};
\node (Input-3) at (0,-2.2) {\raisebox{2mm}{$\vdots$}};
\node (Input-4) at (0,-2.8) {$x_D$};

\node[circle, 
		minimum size = 6mm,
		fill=black!20,
	] (Hidden-1) at (2.5,-1.8) {$\sum w_i x_i \ge t$?};

\foreach \i in {1,...,4}
{
	{
		\draw[->, shorten >=1pt] (Input-\i) -- (Hidden-1);	
	}
}

{            
	\draw[->, shorten >=1pt] (Hidden-1.east) -- ++(1,0)
		node[right]{output $0$ or~$1$};
}
\end{tikzpicture}}
\caption{An artificial neuron}
\label{fig:perceptron}
\end{figure}
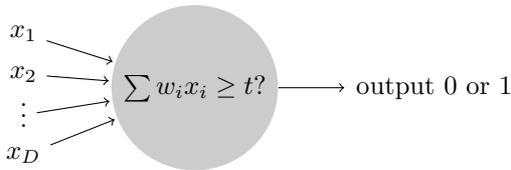

In a geometric interpretation, Equation~\eqref{neuro1:eqn} would mean that the point $(x_1,\dots,x_D)\in\R^D$ is
classified~$1$ if it is \emph{above or on} the hyperplane with normal vector~$(w_1,\dots,w_D)$ and bias $t$, assuming 
an appropriate orientation of the coordinate system.
If $D=2$, this hyperplane becomes the line given by equation
\[ y = \frac{t}{w_2}-\frac{w_1}{w_2}x\ .\]
Replacing $\ge$ by $\le$ in Equation~\eqref{neuro1:eqn} will classify the
points \emph{below or on} the line as~$1$. Although the interpretation of classifying 
points as~$0$ or~$1$ does not depend on the dimension~$D$, we will in this work 
mostly 
study the case~$D=2$ for simplicity.

Networks of artificial neurons (not restricted to  perceptrons), called ANNs, 
are used to approximate (or even solve exactly) 
classification problems in high-dimensional spaces. Formally, a binary classification 
problem is a set of points $S\subseteq \R^D$ and the true classification of 
a point $x\in\R^D$ is simply the membership function. For example, points could be from 
the space of representations of pictures, $S$ could be the set of pictures containing a cat, 
and the classification problem would be to determine whether a given point 
$x\in S'\subseteq \R^D$ is a picture containing a cat. Here $S'$ is an appropriate 
subset of possible queries. 
We call the points in $S$ positive (the others negative) and would like to 
know whether a given point $x\in S'$ is positive or not. In this paper, ANNs with a binary 
output/activation function are used to predict whether $x$ 
is positive (output~$1$) or negative  (output~$0$). The aim is to find a topology and 
parametrization of the ANN that gives the correct prediction on 
as many points from $S'$ as possible. Usually, the degree to which it is achieved 
is measured by the so-called classification error. A classical iterative technique to set 
the weights of ANNs to minimize the classification error is called backpropagation, and 
the underlying iterative process is called ``training'' of the ANN. 
However, backpropagation 
does not straightforwardly work on non-differentiable output functions like
the step function considered here.

In neuro-evolutionary algorithms, here 
again illustrated by the perceptron 
with two input dimensions, 
the search dynamics to minimize 
the classification error usually happens by modifying
the parameters $w_1,w_2$ and $t$ of 
the neuron, which results in moving the decision line associated with the neuron. 
In this paper, we will be dealing with classification problems whose point sets 
are subsets of the unit hypersphere. This motivates us 
to 
use a different representation of the decision line 
corresponding to Hesse normal form, 
\ie, 
by specifying 
angle~$\phi$ of the unit normal vector for the hyperplane (in two dimensions, a line) and its bias~$b$ 
(which then is its distance from the origin measured in the opposing direction of that of the normal vector).
Then the line and the halfspace into which the normal 
vector points are classified
as~$1$.
As 
an additional advantage, the 
parameter set consisting of angle and bias consists only of two values compared 
to three values $w_1,w_2,t$ in the original representation.
The representations $(w_1,w_2,t)$ and $(\phi,b)$ are easily convertible into each other.

\subsection{ANNs with Two Layers}
\label{sec:nneurons}

After having considered the 
single perceptron, we will extend our analysis to ANNs with  a larger number of neurons and layers. 
Here we study a simple structure 
of a so-called feed-forward network 
with two layers, a hidden one
and an output layer.

The hidden layer comprises 
$N>1$ neurons, still with binary 
output function, which are all connected
to the inputs~$x_1,\dots,x_D$. The output 
layer is assumed to compute the 
Boolean OR of the outputs of the hidden
layer. This structure has been chosen 
as we will be dealing mostly with problems 
that can be described as the disjoint union 
of half-spaces of~$\R^D$. See Figure~\ref{fig:twolayer} for 
an illustration.

While the 
hard-wired OR function in the output layer 
may seem problem-specific, 
it is actually not difficult to set  
weights $w_1,\dots,w_N$ and a threshold~$t$ 
that makes the output neuron 
compute a Boolean OR of the binary outputs 
$o_1,\dots,o_N$ from the hidden layer (\eg, 
choose $t>0$ and all weights at least~$t$). 
In experiments (see Section~\ref{sec:experiments}), we placed a
neuron in the output layer which mutated in the
same way as the others. This  
neuron almost always settled to compute
a Boolean function of the outputs of the hidden
layer such that the overall fitness is optimal.
The function was not necessarily an OR, but it
would sometimes  swap the classifications of the two
sides of a line to make them correct. 
A full theoretical 
analysis, where also the parameters 
of the output layer are evolved, is 
subject for future research.

\newcommand{\istart}{-0.95}
\newcommand{\ishift}{0.5}

\begin{figure}[ht]
\centerline{\begin{tikzpicture}
\node (Input-1) at (0,-\istart-\ishift) {$x_1$};
\node (Input-2) at (0,-\istart-2*\ishift) {$x_2$};
\node (Input-3) at (0,-\istart-3*\ishift) {\raisebox{2mm}{$\vdots$}};
\node (Input-4) at (0,-\istart-4*\ishift) {$x_D$};
\foreach \i in {1,2,4}
{
	\node[circle, 
		minimum size = 6mm,
		fill=black!30,
		yshift=18 mm
	] (Hidden-\i) at (2.5,-0.8*\i) {};
}

\node[yshift=18mm] (Hidden-3) at 
(2.5,-2.3) {\raisebox{1mm}{\vdots}};
\node[circle, 
		minimum size = 6mm,
		fill=black!20,			] (Output-1) at (5,-0.0) {\textbf{OR}};

\foreach \i in {1,2,4}
{
	\foreach \j in {1,2,4}
	{
		\draw[->, shorten >=1pt] (Input-\i) -- (Hidden-\j);	
	}
}

\foreach \i in {1,2,4}
{
		\draw[->, shorten >=1pt] (Hidden-\i) -- (Output-1);
}

  \draw[->, shorten >=1pt, dotted] (Hidden-3.east) -- (Output-1);
         
	\draw[->, shorten >=1pt] (Output-1) -- ++(1,0)
		node[right]{out};

\node[below of=Hidden-4, yshift=3mm] {\parbox{2cm}{\centering\small hidden layer\\[-1mm]$N$ neurons}};

\node[below of=Hidden-4, xshift=25mm,yshift=3mm] 
{\parbox{2.5cm}{\small\centering output layer\\[-1mm]$1$ neuron}};
\end{tikzpicture}}
\caption{Structure of ANNs with two layers}
\label{fig:twolayer}
\end{figure}
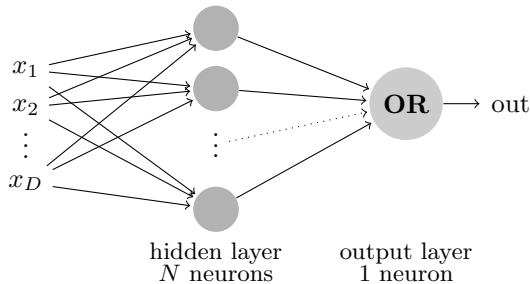

We remark that 
from a theoretical perspective, even 
very small NNs are 
challenging and interesting
to study. 
For example, training networks with 
binary activation functions is  
 NP-hard already for 
 3~neurons 
\cite{BlumRivestNIPS88}. In the case 
of the continuous sigmoidal activation 
function, even the training of a single 
neuron is NP-hard 
\cite{SimaTrainingSigmoidalHard2002}.

\subsection{Algorithms}

Classical frameworks 
for neuroevolution 
like NEAT \cite{stanley2002evolving}  evolve both the topology 
and the weights (and, if applicable, the biases) of the network. This typically leads 
to a mixed discrete-continuous search space, which 
may be hard to analyze. As argued 
above, we assume a fixed topology 
for the network in this first runtime analysis of 
neuroevolution but 
let the evolutionary algorithm evolve the  
weights and biases 
of the neurons. 

Setting the weights and biases of neurons in ANNs is usually a continuous 
optimization problem. However, rigorous runtime analysis is much less developed for continuous optimization problems 
than for discrete problems. 
With the idea of defining a \emph{OneMax of Neuroevolution} in mind, \ie, the transfer of a discrete 
optimization problem, we would like to 
find a discretization of the setting that still represents key aspects of 
the original, continuous setting. 
Here we do not find a 
characterization as pseudo-boolean 
optimization problem $f\colon\{0,1\}^n\to\mathbb{R}$ appropriate.
One could imagine $n$ parallel 
neurons and for each of these 
a binary parameter  that 
corresponds to activating/not activating the 
neuron and count the number of 
activated neurons; however, this 
would be essentially the same as \onemax. The search space $\{0,1\}^n$ 
could also be used to model 
non-binary parameters 
by dividing it into blocks, \eg, of  
$\lceil \log r\rceil$ bits to 
represent an integer in $\{0,\dots,r-1\}$. Again, we do not 
follow this choice since small changes like flipping 
a bit on the genotype (\ie, representation) level might lead 
to large changes in the phenotype. Alternative mappings like grey codes \cite{RothlaufRepresentationsBook}, 
that avoid these 
excessive changes, seem 
hard to analyze.

In this paper, we will use the state 
space $\{0,\dots,r\}^m$, where 
$m$ is proportional to the number 
of neurons and $r$ is the resolution 
of the discretization of a continuous 
parameter from a compact domain; 
\eg, if the parameter lives on $[0,1]$, then the discretization allows 
for the values $[0,1/r,2/r,\dots,1]$. 
Search spaces of the type $\{0,\dots,r\}^m$, where usually $r$ is small but $m$ 
is growing, have been considered 
before in runtime 
analyses, see, \eg, \cite{GuniaGECCO05, DoerrPohlGECCO12, KotzingLWFOGA15, DoerrDKAlgorithmica18}.

After having defined the search space, we must agree on 
the search operators used in the (neuro)evolutionary algorithm. 
A classical search operator in neuroevolution is mutation, 
which may add a Gaussian random variable 
to a weight, choose a weight uniformly 
from a compact interval etc. (both 
of which are valid choices in 
NEAT). We suggest 
 mutation operators that change the 
network parameter (\eg, a weight or bias of a neuron) 
by  
adding $\ell/r$ to the parameter or subtracting $\ell/r$ 
with $\ell$ drawn from some distribution discussed 
below.
Since 
we assume in advance 
that the parameters lives on the compact 
interval like $[-a,a]$ for some constant~$a$  (since bounding the 
domain 
is a typical assumption in the 
optimization of network parameters), we continue 
the interval cyclically, \ie, formally the result 
is taken modulo~$r$ (or modulo~$r+1$). This makes sense especially 
for the  angle of the hyperplane/line belonging to a 
the neuron because of its  periodic structure. See more details below.

In principle, 
the discretized search space $\{-a,-a+2a/r,a+4a/r,\dots,a-2a/r,a\}$ arising from dividing the 
intervals $[-a,a]$ in~$r$ equally spaced segments 
 allows the algorithm 
to find solutions up to an error of $2a/r$ 
in the search space -- not necessarily 
in the objective space. To obtain arbitrary 
precision, real-valued EAs would typically 
introduce a form of self-adaptation (like 
a 1/5-rule, cf.~\cite{JaegerskuepperECJ07}), which we ignore in this 
first study of the runtime of 
neuroevolution algorithms. As a possible 
alternative to our 
discretization,
one could also try to 
work with 
heavy-tailed mutation 
operators for compact 
continuous 
search spaces, but 
these are not easy 
to analyze from a 
theoretical runtime 
perspective, so we 
only consider 
them in the experimental part (Section~\ref{sec:experiments}). Moreover, we 
do not use more advanced search operators like 
crossover here.

Finally, we define the \emph{fitness function~$f$} used 
in the following. Informally this can be understood as the fraction 
of correctly classified points on the unit hypersphere.  
Formally, we consider points  $S_D\coloneqq \{x\in \R^D\mid \ltwonorm{x}=1\}$
as inputs to the ANN and 
 and a binary classification problem with labels 
in $\{0,1\}$ on these points. 
We then compute \[\frac{\vol( ((C_D\cap L_D)\cup 
(\overline{C_D}\cap \overline{L_D}))\cap S_D) }{\vol(S_D)},\]
where $C_D$ is 
the  union of half-spaces 
above (or on) the hyperplanes spanned by the $N$ neurons, $L_D\subset S_D$ the 
set of points classified~$1$, $\overline{A}=\R^D\setminus A$, and $\vol(\cdot)$ denotes the (hyper)volume. As a side note, this 
fitness function can also be understood as the quality of the network with respect to a training set 
uniformly distributed on~$S_D$. 
Note that this training set 
would be of infinite size; in future
work one might want to consider 
a finite-size sample of the set (\eg, 
uniformly), which we think would 
lead to similar results like 
in the present paper but would
involve more corner cases and a 
more involved analysis.

\begin{algorithm2e}
        $t\gets 0$; select $x_0$ uniformly at random         from~$\{0, \dots,r\}^{2N}$.\\
		\While{termination criterion not met}{
  Let $y=(\phi_1,b_1,\dots,\phi_N,b_N)\gets x_t$; \\        
        For all $i\in\{1,\dots,N\}$,  
         mutate 
        $\phi_i$ and $b_i$ with probability $\tfrac{1}{2N}$, independently of each other and other indices;\\
        Mutation chooses  
        $\sigma \in\{-1,1\}$ u.\,a.\,r. and $\ell\sim \harm(r)$ and 
       adds $\sigma\ell$ to the selected component; the result is taken 
        modulo $r$ for angle  and 
        modulo $r+1$ for bias;\\
        For $i\in \{1,\dots,N\}$, set  
        polar angle $2\pi \phi_i / r$ 
        and bias $2b_i/r-1$ for neuron $i$ to evaluate~$f(y)$;\\
		  \lIf{$f(y) \ge f(x_t)$}{$x_{t+1} \gets y$}\lElse{$x_{t+1}=x_t$}
		  $t\gets t+1$;
		}
	\caption{\oona}
	\label{alg:oona}
\end{algorithm2e}

Based on our discussion, we suggest 
the  so-called \emph{(1+1)~Neuroevolution Algorithm} (\oona), given as
Algorithm~\ref{alg:oona}. It maintains $N$ neurons with two inputs each as explained in 
Sections~\ref{sec:oneneuron}--\ref{sec:nneurons}, for which 
biases and angles of normal vectors in polar coordinates are evolved. Recall 
that the representation with angles is preferred over Cartesian coordinates 
because of the spherical structure of our forthcoming benchmark problems. 
We also recall 
that the bias of a neuron coincides with the distance of its line from 
the origin, assuming a unit length for the normal vector.
The algorithm has a global search operator using the harmonic 
distribution~$\emph{Harm}(r)$ on $\{1,\dots,r\}$ 
for the magnitude of change~$\ell$; more precisely, 
$\prob(\ell=i)=1/(i H_r)$ for $i\in\{1,\dots,r\}$, where 
$H_r=\sum_{i=1}^r 1/i$. This operator was used before in \cite{DietzfelbingerRWWAlgorithmica11, DoerrDKAlgorithmica18, HallOSPPSN20} for 
similar search spaces.  
Table~\ref{tab:notation-oona} summarizes the parameters 
and settings of the \oona.

\begin{table}
\centering\begin{tabular}{@{}lp{0.5\columnwidth}}
Notation & Interpretation\\\hline\hline
$N$ & number of neurons\\\smallskip
$D$ & { input dimension for the ANN\newline (usually~$D=2$)}\\\smallskip
$r$ & resolution of angle and bias\\\smallskip
$[-1,1]$ & domain of neurons' bias\\\smallskip
$[0,2\pi)$ & domain of neurons' angle\\\smallskip
$\{0,\dots,r\}^{2N}$ & search space of algorithm\\\smallskip
$(\phi_1,b_1,\phi_2,b_2,\dots,\phi_N, b_N)$ & search point: list of $N$ angle/bias-pairs $(\phi_i,b_i)$, 
 $i\in\{1,\dots,N\}$, for 
the $N$ neurons; $\phi_i,b_i\in \{0,\dots,r\}$\\
$f\colon \{0,\dots,r\}^{2N}\to [0,1]$ & fitness function, returning 
the fraction of correctly classified 
points on the unit hypersphere\\\hline
\end{tabular}
\caption{Overview of notation in \oona (and its analyses)}
\label{tab:notation-oona}
\end{table}

We note that the  independent choices 
for the mutated components 
allow void steps where nothing is mutated. We ignore this here 
for the sake of simplicity; from an algorithm-engineering perspective, one would simply redraw the mutation if it does not change anything \cite{DoerrDLAlgo21,PintoDoerrEA2017}. 

\paragraph{Variants of \oona}
We shall also define and analyze the following 
natural simplifications of the above 
\oona:
\begin{itemize}
\item The \emph{local \oona} only changes 
its components by $\pm 1$, \ie, $\ell=1$ 
is fixed instead of being drawn from 
a Harmonic distribution. This operator 
is called \emph{unit mutation} in 
\cite{DoerrDKAlgorithmica18}. 
\item 
 The \emph{\oona without bias} 
 fixes $b_i=0$ for 
    $\{1,\dots,N\}$ and maintains search points $(\phi_1,\dots,\phi_N)$ 
    consisting of $N$ angles only, which 
    are subject to the same type 
    of mutation as the original \oona.
\end{itemize}

Further variants of the algorithm
may be investigated in the future, \eg, 
different 
choices
for the mutation and other probabilistic elements.
In 
particular, the self-adjusting mutation from~\cite{DoerrDKAlgorithmica18}  is a  rather 
relevant alternative to the simple 
mutations considered here.

\paragraph{Optimization time} A common convention in 
runtime analysis in discrete search spaces is to define the optimization time (synonymously, \emph{runtime}) as the number of 
fitness function evaluations until a solution having optimal 
fitness value has been sampled. We adapt this to our discretized search 
space by still counting fitness evaluations, but saying that the function has been optimized if the current 
search point $x=(\phi_1,b_1,\dots,\phi_N,b_N)$ deviates by an absolute value of less than~$1$ in 
the representation of angles and biases, \ie, 
\[
\max_{i=1}^N\{\abs{\phi_i - \phi_i^*},\abs{b_i -b_i^*}\}< 1
\]
for an optimal (fractional) solution 
$(\phi_1^*,b_1^*,\dots,\phi_N^*,b_N^*)$, where the absolute values are with wrap-around in the respective  intervals. This corresponds to an $O(1/r)$-error in 
terms of the actual value of bias or angle. 
Typically, the expected value of the stochastic optimization time
is bounded. Since each element of the search space has a probability of at least $(1/(rH_r))^{2N}$ of 
being hit by mutation, we obtain the following bound, similar to the worst-case 
bound for the \ooea on pseudo-boolean problems \cite{DJW02}.

\begin{lemma}
\label{lem:general-upper}
The expected optimization time of the \oona on an arbitrary problem is 
at most $O((r\log r+r)^{2N})$.
\end{lemma}

Even if $N=c$ for a constant~$c$, we do not consider the general runtime bound $O((r\log r+r)^{2c})$ as particularly  
efficient. In fact, 
for simple problems  bounds being polylogarithmic in $r$ like $(\log r)^{O(1)}$ can be obtained, as shown and discussed below.

Finally, we remark that the set of optimal solutions for a given optimization problem may depend on the 
number of allowed neurons and whether bias is allowed or not. We will consider examples 
where with only one neuron, not all points of the underlying classification 
problem can be classified correctly, while this is possible with at least two neurons.

\subsection{Problems}
\label{sec:problems}
In this section, we define several 
benchmark problems that shall illustrate 
how the \oona makes progress towards 
a correct classification with 
one or several neurons. Also, the 
section serves to point out typical 
situations in the optimization that 
can make the algorithm stuck in 
a local optimum. As argued 
above, 
we identify problems with the points 
in $S_D\subseteq \R^D$ classified positively, \ie, as~$1$. 
All 
problems are defined for arbitrary~$D\ge 2$; 
however, for the sake of simplicity 
most analyses will be restricted 
to~$D=2$. 

The following problem 
can be thought of as a kind of \onemax  
for the \oona without bias. However, there 
are limits to this analogy since 
the fitness landscape for the \oona 
(see Section~\ref{sec:runtime-analysis}) is 
more uniform than the for the classical 
\ooea on \onemax.

\begin{definition}
The problem \halfsphere consists 
of all points with non-negative $x_D$-dimension on the unit hypersphere, \ie,  
\begin{align*}
\halfsphere & = 
\{x\in \R^D \mid 
\ltwonorm{x} = 1 \text{ and } x_D \ge 0\}\\
& = 
\{x\in \R^D \mid 
\ltwonorm{x} = 1 \text{ and } \psi_{D-1} \in [0,\pi]\},
\end{align*}
where $\psi_{D-1}$ is the polar 
spherical angle between $x$ and the 
unit hypersphere on the first~$D-1$ 
dimensions.
\end{definition}

Obviously, setting the angle of a single neuron to $\pi/2$ and its bias to~$0$ is optimal here. See Figure~\ref{fig:illustration-spheres} for sketch of \halfsphere and the following 
two problems with $D=2$.

\begin{figure}[ht]
\centering
\begin{tikzpicture}
\draw (0,0) circle (1);
\draw[red, ultra thick] (1,0) arc (0:180:1);
\draw[blue,thick] (-1.1,0)--(1.1,0);
\draw[black,->] (0,0) -- (0,0.4);
\draw[black,->] (0,0) -- (0,0.4););
\begin{scope}[shift={(3,0)}]
\draw (0,0) circle (1);
\draw[red, ultra thick] (1,0) arc (0:90:1);
\draw[blue,thick] (-0.1,1.1)--(1.1,-0.1);
\draw[black,->] (0.5,0.5) -- (0.9,0.9);
\end{scope} \begin{scope}[shift={(6,0)}]
\draw (0,0) circle (1);
\draw[red, ultra thick] (1,0) arc (0:90:1);
\draw[red, ultra thick] (-1,0) arc (180:270:1);
\draw[blue,thick] (-0.1,1.1)--(1.1,-0.1);
\draw[black,->] (0.5,0.5) -- (0.9,0.9);
\end{scope} \end{tikzpicture}
\caption{Illustration of \halfsphere, \quartersphere and \twospherequarters (from left to right). The thick red areas constitute 
the target points that should be classified positively. Hyperplanes at optimal positions
are shown in blue. The arrow points to the positively classified halfspace. The fitnesses are $1$, $1$, and $3/4$, respectively.}
\label{fig:illustration-spheres}
\end{figure}
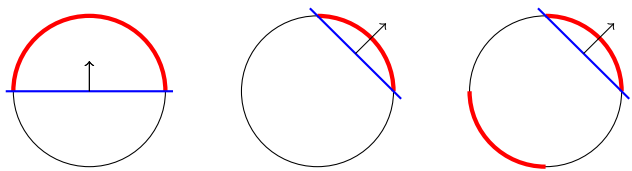

Similarly, we define
\quartersphere. It can still be solved optimally (according to our definition 
above) 
with the special case of 1~neuron 
if the bias is allowed to vary. The global optimum is at angle $\pi/4$ and bias $\sqrt{2}/2$. 

\begin{definition}
The problem \quartersphere consists 
of all points with non-negative $x_{D-1}$ and $x_D$-dimension on the unit hypersphere, \ie,  
\begin{align*}
\quartersphere & = 
\{x\in \R^D \mid 
\ltwonorm{x} = 1 \text{ and } (x_{D-1},x_D)\ge (0,0)\}\\
& = 
\{x\in \R^D \mid 
\ltwonorm{x} = 1 \text{ and } \psi_{D-1} \in [0,\pi/2]\},
\end{align*}
\end{definition}

\tikzset{    add/.style args={#1 and #2}{
        to path={ ($(\tikztostart)!-#1!(\tikztotarget)$)--($(\tikztotarget)!-#2!(\tikztostart)$)  \tikztonodes},add/.default={.2 and .2}}
}  

The next problem requires at least 
two linear classifiers, \ie, two neurons, and a neuron (possibly 
hard-wired) joining the results of those two to be solved 
exactly. With only one neuron, at least~$1/4$ of the 
circle will be classified incorrectly.

\begin{definition}
The problem \twospherequarters consists 
of all points with either both non-negative or both positive $x_{D-1}$ and $x_D$-dimen\-sion on the unit hypersphere, \ie,  
\begin{align*}
& \twospherequarters  = 
\{x\in \R^D \mid 
\ltwonorm{x} = 1 \text{ and } x_{D-1}x_D\ge 0 \}\\
& \quad\qquad = 
\{x\in \R^D \mid 
\ltwonorm{x} = 1 \text{ and } \psi_{D-1} \in [0,\pi/2]\cup[\pi,3\pi/2]\},
\end{align*}
\end{definition}

Finally, we define the problem that 
has hard-to-overcome 
local optima in the 
fitness landscape given by the local \oona. 
See more details in the following section. It is most convenient
to define the problem by means 
of polar coordinates.

\begin{definition}
The problem \spherelocalopt consists 
of all points  on the unit hypersphere with polar angle $\psi_{D-1}$ between $0$ and $60$, $120$ and $180$,  or 
$240$ and $330$ degrees, \ie,  
\begin{align*}
& \spherelocalopt 
 = 
\{x\in \R^D \mid 
\ltwonorm{x} = 1  \\
& \qquad \text{ and } \psi_{D-1} \in [0,\pi/3] 
\cup [2\pi/3,\pi] \cup 
[4\pi/3,11\pi/6]\}
\end{align*}
\end{definition}

Figure~\ref{threarc:fig} shows examples of a globally optimal solution and one in a local optimum, assuming $D=2$.
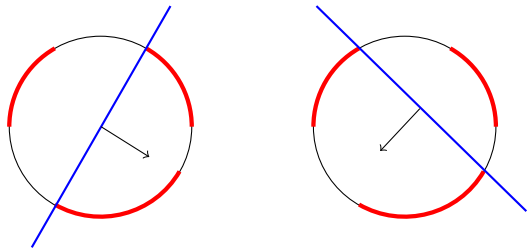
\begin{figure}
  \centering
  \begin{tikzpicture}[scale=0.8]
\draw (2,2) circle [radius=1.5];
\draw[red, ultra thick] (3.5,2) arc (0:60:1.5);
\draw[red, ultra thick] (0.5,2) arc (180:120:1.5);
\draw[red, ultra thick]  (1.25,0.71) arc (240:330:1.5);
\draw [blue,thick] (0.87,0)--(3.15,4);
\draw [black,->] (2.01,2)--(2.8,1.5);
\begin{scope}[shift={(5,0)}]
\draw (2,2) circle [radius=1.5];
\draw[red, ultra thick] (3.5,2) arc (0:60:1.5);
\draw[red, ultra thick] (0.5,2) arc (180:120:1.5);
\draw[red, ultra thick]  (1.25,0.71) arc (240:330:1.5);
\draw [blue,thick] (0.55,4.0)--(4.0,0.60);
\draw [black,->] (2.25,2.3)--(1.6,1.6);
\end{scope}
\end{tikzpicture}
  \caption{Examples for \spherelocalopt. 
  The colors are as in figure~\ref{fig:illustration-spheres}.
      One optimal solution is shown at the left having  a fitness of
  $3/4$ (there are two more). The solution at right is locally optimal with a
  fitness of $2/3$.} 
  \label{threarc:fig}
\end{figure}

To conclude this section and to prepare the forthcoming analyses, we define a problem class that all above problems belong to.

\begin{definition}
\label{def:unionofarcs}
A classification problem $S\subseteq \R^D$ is called \emph{union of (generalized)  arcs} if there are two constants $n^*\in\N,a^*\in \R^{\ge 0}$ such 
that~$S$ is the disjoint union of $n^*$ hyperspherical caps~$C_1,\dots,C_{n^*}$ 
of the unit hypersphere, where 
each cap~$C_i$,
$i=1,\dots,n^*$, 
 is given by 
\[C_i = \{x\in \R^D \mid 
\ltwonorm{x} = 1 \text{ and } \psi_{D-1} \in [\alpha_{i},\beta_i]\}
\]
for $\alpha_i,\beta_i\in\R$ with $\beta_i-\alpha_i\ge a^*$, 
\ie, is defined on a  
interval  of constant 
size 
with respect to the polar 
angle $\psi_{D-1}$. 
\end{definition}

Since we mostly will work with $D=2$, we 
prefer the term \emph{union of arcs} 
instead of the generalized ``union of hyperspherical caps'' or similar in the following.
Obviously, from the problem definitions 
it immediately follows that  the problems defined above
are all 
union of arc problems.

\section{Runtime Analysis}
\label{sec:runtime-analysis}

In this section, we conduct rigorous runtime 
analyses of the \oona with harmonic and local mutation
on the example problems 
defined in Section~\ref{sec:problems}. 
We exclusively consider the case~$D=2$ 
here, which justifies the use of the 
terms for 2~dimensions like ``unit circle'', ``line'' etc. instead of the general 
``unit hypersphere 
(surface)'', ``hyperplane'' etc. 
Extensions of the analyses to larger 
dimensions
seem promising, but would require 
inputs of higher dimensionality for the neurons and would introduce additional 
complexity.

We start our analysis with the simplest 
of the problems, which can be 
solved optimally (in the sense 
defined above) with $1$ neuron even if the bias is 
fixed at~$0$.

\begin{theorem}
\label{theo:halfsphere}
The expected optimization time 
of the  \oona with harmonic mutation, $N=1$ and 
without bias on 
\halfsphere for $D=2$ is 
$O(\log^2 r)$. For 
the local \oona it is $O(r)$.
\end{theorem}

To prove this and the following theorems,
it is crucial to understand how the 
\oona can make progress towards 
solutions of higher fitness. To this 
end, we give the following definition
and characterization of local optima.

\begin{definition}
\label{def:localopt}
Let $x_t=(\phi_1,b_1,\dots,\phi_N,b_N)$ be a search point of 
the \oona. We call $x_t$ a \emph{local optimum} if
there is no $x'$ of strictly larger 
fitness value that differs from~$x_t$ 
in exactly one component~$c$ and furthermore 
by either $+1$ or~$-1$ 
in that component 
(modulo~$r$ if $c$ denotes an angle 
and modulo~$r+1$ if $c$ denotes a bias).
\end{definition}

\begin{lemma}
\label{lem:charac-local-opt}
Consider the \oona with $N=1$ on a unit-of-arcs 
problem according to Definition~\ref{def:unionofarcs}. Assume 
for the current search point~$x_t$  
that 
not both intersecting points of the unit circle 
and the neuron's  
line 
are at a boundary of a positive or negative 
arc. If at least one of the two following conditions
hold, then $x_t$ is not locally optimal: (1) none 
of the intersecting points is at a boundary 
and the smallest distance between the 
endpoints of a positive arc and the 
intersecting points is at least~$2/r$; (2) one 
intersecting point is at a boundary and there 
is a 
negative arc of length at least~$2\pi/r$ 
incident on the other intersecting point.

The lemma still applies to the case $N>1$ 
if all lines belonging to the $N$ neurons 
classify disjoint areas of $[0,1]^2$ 
positively, have constant distance 
from each other, and $r$ is at least 
a sufficiently large constant.
\end{lemma}

\begin{proof}We recall that the unit circle is composed 
of a constant number of arcs of constant 
size, where each arc is either 
classified completely  positively or completely 
negatively.

We consider the two cases stated in the lemma. If none of the 
intersecting points is at a boundary, then 
it is sufficient to change the bias component of the current search point
by $\pm 1$, moving the line either closer to a 
positive arc (if it is above the line) 
or further away (if below). This reduces the 
length of the wrongly classified region and 
does not move the line into a different 
arc by the assumptions of distance at least $2/r$. 

If only one intersecting 
point is at a boundary, then, by 
the assumption on the negative 
arc length of at least 
$2\pi/r$, changing the 
angle of the line is possible without 
decreasing the fitness. More precisely, 
we change the angle such that the intersecting point, which previously was at 
a boundary, now lies within a negative arc. 
If the angle 
changes by $\delta$, this increases the 
negative length above the line by $\delta$ (since we are dealing with a unit circle). However, the other intersecting point moves by the same amount closer to a 
positive arc. See Figure~\ref{fig:quartersphere-rotate} 
for an illustration.

The conclusions above are still valid 
for larger~$N$ if moving or rotating a line 
preserves disjointness of 
the $N$ arcs classified positively by the neurons. If $r$ is bounded from below by a sufficiently large constant, this 
holds since we 
assume at least constant distance 
between the~$N$ lines (measured within the unit circle).
\end{proof}

With the above tools at hand, we can 
prove the first theorem.

\begin{proofof}{Theorem~\ref{theo:halfsphere}}
By definition, the optimal angle for the problem is~$\pi/2$, resulting in the halfspace $\{x\in \R^2\mid x_2 \ge 0\}$ 
being classified positively. (Here we use that the bias is fixed at~$0$.) We 
consider the local \oona first and 
use 
a classical fitness-level argument \cite{WegenerICALP01} for the 
underlying unimodal fitness landscape 
to bound the time until the current angle 
of the \oona 
has reached~$\pi/2$, corresponding to 
$\phi_t=r/4$ in the search space (note that with our definition of optimization time, any angle $\phi_t\in (r/4-1,r/4+1)$ would 
be considered as optimal). Let $\xi_t=
\min\{\card{r/4-\phi_t}, 5r/4 -1-\phi_t\}$, \ie, the smallest distance of $\phi_t$ from 
its optimum~$r/4$ in the representation with wrap-around, where $0$ is a neighbor of $r-1$.

If $\xi_t= i >0$ (corresponding, \eg, to an angle~$2\pi i/r+\pi/2$), then 
incrementing or decrementing~$\phi_t$ by~$1$ (modulo~$r$) will 
improve fitness since such a step increases the length of the intersection of the
arc  above 
the neuron's line 
(with normal vector of angle $\phi_t$)
and the  points in \halfsphere. Whether increasing or decreasing (or 
both) improves fitness depends on whether $\phi_t < 3r/4$. The local \oona 
chooses the improving direction for the angle 
with probability at least~$1/2$ and reduces 
$\xi_t$ by~$1$ with probability at least~$1/2$. Altogether, the probability of 
improving is at least~$1/4$. Since at most $r/2$ 
improvements are sufficient, the total expected optimization time is at
most $(r/2)\cdot 4 = O(r)$.

For the standard \oona with Harmonic 
mutation, we use multiplicative 
drift analysis \cite{DoerrJWAlgo12}, inspired by the 
analysis of the Harmonic mutation 
on a generalized \onemax function from \cite{DoerrDKAlgorithmica18}. Let 
$\xi_t=i$. 
Then all decreasing steps 
of sizes $1,\dots,i$ are accepted. The 
probability of a decreasing step of size~$j\le i$ is $1/(2jH_r)$, where the
factor~$2$ accounts for the choice of direction. 
Hence, the expected distance decrease 
is at least 
$\sum_{j=1}^i \frac{1}{2jH_r} \cdot j 
= \frac{i}{2H_r};
$
in other words, the drift is bounded from below by 
$i\delta$ with $\delta=1/(2H_r)$. By the multiplicative 
drift theorem \cite{DoerrJWAlgo12}, the expected hitting time 
of $0$ is $O((\ln \phi_0 + 1)/\delta)
= O(\ln^2 r)$.
\end{proofof}

We remark that the optimization problem 
underlying Theorem~\ref{theo:halfsphere} 
corresponds to the generalized \onemax on 
$\{0,\dots,r\}^n$ for $n=1$ as considered in 
\cite{DoerrDKAlgorithmica18}. Hence, it 
seems straightforward to transfer their 
results for an advanced 
self-adjusting mutation to 
the scenario of Theorem~\ref{theo:halfsphere}. However, 
this is not obvious for the following 
problems, so we stick to the more simple local and harmonic mutation 
operators for the rest of 
this paper.

We proceed now to the problem \quartersphere,
which cannot be solved optimally with 
one neuron if the bias stays fixed at~$0$. 
We show that allowing varying bias leads 
to polynomial in~$r$ expected optimization time 
for the local mutation and even to
polylogarithmic times in special cases for the 
Harmonic mutation. We remark that the positive 
arc length of $\pi/2$ has been chosen for 
simplicity. The analyses also  
hold for a larger class of problems where the 
length of the positive arc is some constant value in the
interval $(0,\pi)$.

\begin{theorem}
\label{theo:quartersphere}
Let $r=8k$ for an integer~$k$, let $N=1$ and allow
variable bias. Consider the   \oona with local 
and harmonic mutation on the problem \quartersphere. Then:
\begin{enumerate}
\item 
The expected optimization time of the local 
\oona is $O(r^2)$.
\item With 
at least constant probability, the 
optimization time of the local \oona is $O(r)$.
\item 
 With harmonic mutation, 
 the expected optimization  
 time is $O(\log^3 r)$.
 \end{enumerate}
\end{theorem}

We will need the following characterization of fitness, assuming 
that there is piece of positive length  above the line.

\begin{lemma}
\label{lem:charac-fitness-quartersphere}
    Let $x_t=(b_t,\phi_t)$ be the current search point of the \oona with $N=1$ on 
    \quartersphere and assume that there is piece of positive length on the unit circle  
    above the neuron's line.  Let $\db{t}=(2b_t/r-1)-\sqrt{2}/2$ and $\dphi{t}=\card{2\pi \phi_t/r-\pi/4}$ be the difference of current bias and absolute difference of angle, respectively, from their optimal values and let 
    $\eta_t = 2\arccos(2b_t/r-1)-\pi/2$ be the difference of the length of the arc above the line 
    from its optimum value.  If $\db{t}\ge -\sqrt{2}/2$, then it holds for 
    the current fitness value that
    \[
    f(x_t) = 1 - \card{\eta_t} - \max\{0,\dphi{t }-\card{\eta_t}/2\}.
    \]
\end{lemma}

\begin{proof}By simple trigonometry, the length of the arc above the line is $2\arccos(2b_t/r-1)$. We distinguish 
two cases according to $\db{t}$.
If $\db{t}\le 0$, which means that the whole positive arc of \quartersphere 
has room to be lie completely above the line,  a total length of least $\eta_t 
= 2\arccos(2b_t/r-1)-\pi/2 \ge 0$ of that arc is negative and therefore 
wrongly classified. However, if $\dphi{t} > \eta_t/2$, then the positive arc of \quartersphere intersects the line 
and an additional arc of length $\dphi{t} - \eta_t/2$ lies 
below the line and is wrongly classified. This gives the formula 
 \[
    1-f(x_t) = \eta_t + \max\{0,\dphi{t }-\eta_t/2\}.
    \]
    If $\db{t}>0$, then $\eta_t<0$ and 
    at least one endpoint of the positive arc 
    of \quartersphere is below the line.   
    Moreover, positive arcs of total length 
    at least $-\eta_t=\pi/2-2\arccos(2b_t/r-1)$ are below the line and 
    wrongly classified, where we use that the
    bias is non-negative due to the assumption 
    $\db{t}\ge -\sqrt{2}/2$. Together with the additional negative arc 
    above the line in the case $\dphi{t}>\card{\eta_t}/2$, we obtain
     \[
    1-f(x_t) = -\eta_t + \max\{0,\dphi{t }+\eta_t/2\},
    \]
    and the lemma follows.
\end{proof}

We shall show the statements of the 
Theorem~\ref{theo:quartersphere} separately and start with the simpler case 
of local mutation analyzed in the first two statements.
The second statement 
considers 
an initialization where there is a clear gradient towards 
on optimal solution. The first statement considers general 
initialization, which may result in a longer random-walk behavior with the line of the neuron 
being tangent on the unit circle.  

For all parts of Theorem~\ref{theo:quartersphere}, we will frequently use 
the following helper lemma. Often the line~$\ell$ considered 
in the lemma corresponds to an optimal 
placement of the neuron's line.

\begin{lemma}
\label{lem:point-move}
Let $\ell$ be a line passing through the unit circle and $p$ be a point on the unit circle. 
Let $d_\ell$ be the distance of $\ell$ from the origin and 
let $d_p$ be the distance between the origin and the line that is parallel to $\ell$ and passes through $p$. Assume that both $d_\ell$ and $d_p$  less than~$1$.

Suppose that $p$ is rotated by an angle of $\rho$ on the circle  such that it moves either closer to~$\ell$ or further away 
from~$\ell$ 
during the whole rotation 
and does not change side \wrt~$\ell$. Then there are constants~$0<c_1<c_2$ such that 
 $d_p$ reduces by at least $c_1\rho$ and at most $c_2\rho$. Similarly, if $\ell$ is moved closer to $p$ by an amount 
of~$\delta>0$, then the arc between~$\ell$ and~$p$ 
decreases by at least $c_3\delta$ and at most $c_4\delta$ 
for constants $0<c_3<c_4$.
\end{lemma}

\begin{proof}The distance $d_p$ is given by the sine of the angle~$\alpha$ 
between $\ell$ and the line passing through the center 
point and~$p$. 
Now, since moving the point decreases 
the angle by a constant, 
the first claim follows by noting that sine is monotone increasing in the considered ranges and that its derivative is constant 
as soon as the angle has moved away from $\pm \pi/2$. 
The second 
claim follows analogously with 
a linear approximation of the arcsine.
\end{proof}

We now give the proofs of the second statement and afterwards of the first statement 
of Theorem~\ref{theo:quartersphere}. 

\begin{proofof}{2nd statement of 
Theorem~\ref{theo:quartersphere}} 
This statement considers beneficial initializations of angle and biases. If the set \quartersphere (an arc of length~$\pi/2$)  intersects 
or lies completely above 
the line of the neuron in the initial solution, then fitness is improved by rotating the line to increase the  arc length above the line or 
increasing the bias (or both). We consider the event of an initial bias in $(0,0.5]$ (\ie, $b_0\in  [r/2+1,3r/4]$ in the initialization) and 
an initial angle strictly in between  $a\coloneqq \arcsin(0.5)$ and $\pi/2-a$  (roughly corresponding 
to $\phi_0\in (0.0833r,0.1667r)$), which has constant probability. 
Simple geometry 
then shows that the \quartersphere 
arc is completely above the line of the neuron. This implies that the points in 
\quartersphere and the region below the line, which comprise  
more than half of the circle because of the positive bias, are correctly classified.
Hence, the fitness of such a search point is 
strictly larger than $3/4$. It is not possible to achieve such a fitness without having 
the \quartersphere arc partially or fully above the neuron's line, so this property 
will be maintained during the run.

To analyze fitness improvements, we consider two types of steps: 
\begin{enumerate}
 \item increasing the bias, \ie, moving the line closer to the positive arc of \quartersphere without moving any of its points below the line, 
 \item 
 changing the angle such that the arc of \quartersphere appears 
 more centered above the line of the neuron; formally, 
 the absolute difference between angle and $\pi/4$, \ie, the 
 quantity $\dphi{t}$ from Lemma~\ref{lem:charac-fitness-quartersphere} decreases. This 
 may be necessary to allow a further increase in bias and fitness. 
Figure~\ref{fig:quartersphere-rotate} depicts
a situation where a type-2 step has to be applied before further improvements.
\end{enumerate}

The characterization of Lemma~\ref{lem:charac-local-opt} shows that type-1 
steps or type-2 steps are available before the 
line has found its optimal position (up to the 
allowed tolerance~$\pm O(1/r)$). 
Moreover, type-1 steps strictly improve 
fitness (unless the bias component of the 
search point has reached the optimum value $\pm 1$).  
As long as type-1 steps are available (\ie, can improve fitness), 
there is  
a probability of at least $\Omega(1)$ of
not changing the angle and mutating the bias in the 
desired direction (increasing it by~$2/r$). By Lemma~\ref{lem:charac-fitness-quartersphere}, this increases the fitness since $\eta_t$ 
decreases. 
This fitness 
improvement is at least $\Omega(1/r)$ since the length of wrongly 
classified 
region above the line decreases linearly with the 
increase of bias by Lemma~\ref{lem:point-move}.
Here we exploit the bias considered 
here is at least~$0$ and at most $\sqrt{2}/2+1/r$, which is by 
a constant away from~$1$, \ie, the radius of the 
unit circle.

When a type-1 step is not available, a 
type-2 step may decrease~$\dphi{t}$.
Such steps 
do never decrease fitness 
and are accepted, however, they do not 
necessarily increase fitness 
if \quartersphere is already completely 
above the line. If 
neither type-1 nor type-2 steps are available, 
both angle and bias are within an 
additive distance of less than~$1$ from the optimum in the representation. 
This holds since we assume that $r$ is a multiple of~$8$, so the optimum 
angle of $\pi/4$ 
can be represented as $r/8$ in the search point. Hence, when the angle takes
precisely its optimum value, the bias may decrease to its optimum $\sqrt{2}/2$
within the error $O(1/r)$ introduced by the discretization. This is the
desired state that we analyze the 
algorithm to reach.

If a type-2 step is available that increases fitness, then 
we are in the situation $\dphi{t}\ge \eta_t/2$ of 
Lemma~\ref{lem:charac-fitness-quartersphere} and the fitness 
improvement is at least~$1/r$ (using again that $r$ is a multiple 
of~$8$, so $\dphi{t}$ 
cannot take values strictly in  between~$0$ and $1/r$).
Finally, we have to analyze  
fitness progress in the situation 
that only type-2 steps are available that do not 
decrease fitness since $\dphi{t}\le \eta_t/2$. A type-2 step may 
decrease $\dphi{t}$ and thereby increase the 
distance of closest endpoint of the positive \quartersphere 
arc from the 
 neuron's line. 
If this distance increases by at least~$1/r$, 
a type-1 step becomes available. 
Hence, using again Lemma~\ref{lem:point-move} and noting that old and new bias are bounded by a constant less than~$1$ and at least~$0$, rotating the angle by 
at least $c/r$ for a sufficiently large constant~$c$ increases 
the distance of the closest endpoint of the  \quartersphere arc from the line 
of the neuron by at least~$2/r$. Thus, 
$c$ consecutive steps rotating in 
one direction are sufficient, which 
happens with constant probability at least $(1/4)^c$.
Altogether, there is a constant probability of changing
the angle by at least~$c/r$ in the desired direction
and afterwards increasing the bias by~$2/r$, 
\ie, a sequence of type-2 steps and a type-1 step. Hence, as long the 
algorithm 
is not yet within the allowed distance of the optimal bias, there is 
a constant probability of improving fitness 
by $\Omega(1/r)$. Altogether, in expected $O(r)$ steps 
the \oona has reached the desired state.
\end{proofof}

\begin{proofof}{1st statement of 
Theorem~\ref{theo:quartersphere}} 
Again, we use the ideas from Lemma~\ref{lem:charac-local-opt}.  
Hence, there is always a constant probability of increasing the 
fitness unless  either the global optimum has been reached (in the sense 
of an error of~$O(1/r)$ as described above) or the bias has reached
its maximum value~$1$. If a part of 
\quartersphere is above the line, then fitness 
increases by moving a larger part above the line without 
moving too many of the negative points below it. If, however, \quartersphere is 
completely below the line, then fitness increases by decreasing the 
length of the arc  above the line (consisting of negative points only). Hence, an increase of bias increases fitness in this situation, up to the point 
where it reaches~$1$ 
and, up to an intersection of volume~$0$, the whole unit circle is below the line. 
Such a state corresponds to 
a fitness of $3/4$. Decreasing the bias is not accepted in this situation unless it moves parts of the positive \quartersphere arc 
above the neuron's line. As soon as this happens, fitness increases above $3/4$ 
and an optimum is bound in expected time $O(r)$ by decreasing bias and rotating the line until reaching bias $\sqrt{2}/2\pm O(1/r)$ and angle~$\pi/4$, using the same 
arguments as in the analysis of the first statement. Also, again using the same 
arguments, in expected time~$O(r)$ the bias 
reaches~$1$ if the optimum is not found before. Hence, we only have to analyze 
the time, starting from a bias of~$1$, until the 
angle enters the interval $(0,\pi/4)$. In this 
situation there is a constant probability of 
decreasing the bias (without changing the angle). 

To complete the proof, we consider the random walk of the angle while the 
bias is~$1$. Formally, let $X_t$ be the representation of the angle at 
time~$t\ge 0$ (\ie, as an integer in $\{0,\dots,r-1\}$), 
where time~$0$ corresponds to the first point in time with bias~$1$. We consider the first hitting time~$T\coloneqq \min\{t\ge 0\mid X_t \in [0,r/4]$, assuming $X_0\in (r/4,r-1]$. 
If the bias does not change, 
the random walk takes independently in each step a uniform decision to 
increase or decrease the angle, and the absolute 
change is independently drawn from the same distribution. 
Hence, we have 
$\expect{X_{t+1}\mid X_t} \le X_t$, 
\ie, a supermartingale, where the inequality stems  that
the mutation is 
taken modulo~$r$. We pessimistically assume the case of a martingale. 
Clearly, since the change has constant 
variance and uniform random sign, 
the variance satisfies $\Var{X_{t+1}\mid X_t}\ge c$ for 
a constant~$c>0$. Hence, by the upper bound for 
martingale drift (Corollary~26 in~\cite{KotzingKrejcaTCS19}), 
we have $\expect{T}=O(r^2)$. 
By Markov's inequality, $T=O(r^2)$ with 
constant probability, and with altogether 
constant probability after $O(r^2)$ steps 
the bias decreases to  less than~$1$.

Finally, the total expected runtime is $O(r^2)$ by 
a standard restart argument. Formally, 
we can consider independent phases of length
$O(r^2)$ and constant success probability. 
The expected number of such phases is $O(1)$.
\end{proofof}

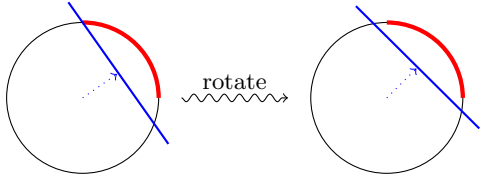
\begin{figure}[ht]
\centerline{\begin{tikzpicture}
\draw (0,0) circle (1);
\draw[red, ultra thick] (1,0) arc (0:90:1);
\node (A) at (90:1) {};
\node (B) at (340:1) {};
\draw [add= -1.2 and -1.2, blue, thick] (A) to (B);
\draw[blue,->,dotted] (0,0) -> (35:0.55);
\draw[->, decorate, decoration={snake, segment length=2mm, amplitude=-0.5mm}] (1.3,0)  -- (2.7,0) node[pos=0.5,above] {\small rotate};
\begin{scope}[shift={(4,0)}]
\draw (0,0) circle (1);
\draw[red, ultra thick] (1,0) arc (0:90:1);
\node (C) at (100:1) {};
\node (D) at (350:1) {};
\draw [add= -1.2 and -1.2, blue, thick] (C) to (D);
\draw[blue,->,dotted] (0,0) -> (45:0.55);
\end{scope} \end{tikzpicture}}
\caption{The line  
can be rotated without 
changing fitness.}
\label{fig:quartersphere-rotate}
\end{figure}

To complete the analysis, we still have to analyze
 the standard \oona with the global, harmonic  mutation. While several ideas 
from the case with local mutation will reappear, the analysis 
is  more complex and requires a careful study 
of the decomposition 
of the fitness value from Lemma~\ref{lem:charac-fitness-quartersphere}, along 
with more advanced drift arguments. We will need the following helper statement.
\begin{lemma}
\label{lem:harmonic-interval}
    Let $X$ denote the random outcome of the harmonic mutation 
    operator with parameter~$r$. Let $a,b$, where $a<b$, be two positive 
    integers. Then 
    $\prob(a< X\le b) \ge (\ln(b/a)-1/a)/H_r$
\end{lemma}

\begin{proofof}{Lemma~\ref{lem:harmonic-interval}}
    We compute \begin{align*}
    \prob(a< X\le b) & = \sum_{j=a+1}^b \frac{1}{jH_r} 
    = \frac{1}{H_r}(H_b-H_a)\\
    & \ge \frac{\ln(b)-\ln(a)-\tfrac 1 a}{H_r} 
    = \frac{\ln(b/a)-\tfrac 1 a}{H_r},
    \end{align*}
    where we bounded the Harmonic sums by integrals using upper and lower sums.
\end{proofof}

Lemma~\ref{lem:harmonic-interval} is 
often used for the case that $b-a=\Omega(r)$, 
which gives a probability of $\Omega(1/\ln r)$ of hitting the interval $(a,b]$. 
We can now give the analysis of the harmonic mutation on \quartersphere.

\begin{proofof}{3rd statement of Theorem~\ref{theo:quartersphere}}
First, let us assume that we already have a fitness of at least~$3/4$ to 
explain the main idea. At the end of the proof, we will deal with 
a general initialization.
 
In the following, we study 
the effect of the global, harmonic mutation in the 
analysis of the fitness improvements and progress of the line towards its optimal state at bias 
$b^*\coloneqq \sqrt{2}/2$ and angle~$\pi/4$. Note that the  line might change 
globally in one step, from a state with bias less than~$b^*$ 
and the whole positive arc and some 
negative parts above it to a state with bias greater than~$b^*$ and a 
strict subset of the positive arc and no negative 
points above it. However, the assumption of  
fitness strictly larger than~$3/4$ 
gives us two invariants: (1) 
 the bias must be non-negative (otherwise, more than $1/4$ of the negative parts 
 would be above the line and classified 
 incorrectly) and (2)
that there will always be some 
positive (sub)arc of \quartersphere above the 
line (as already exploited in the proof 
of the 2nd statement).

To analyze fitness improvements, we are in the 
same mindset as in the proof of 
the first two parts 
and distinguish 
between changes of angle and changes of bias. 
Changes of angle are beneficial if the bias is 
less than~$b^*$ and the positive arc is partly below 
the line. Let~$\xi$ be length of positive arc part below 
the arc. Then rotating the line by~$\xi$ immediately 
improves fitness by~$\xi$.
Changes of bias 
become relevant if 
both end points of the positive arc are on the same 
side of the line. Then fitness improves by 
changing bias, 
either by bringing more positive points (if bias 
is decreased) or fewer negative points (if bias is 
increased) above or on the line.

With this rough strategy in mind, 
we shall study the fitness distance $g_t=1-f(x_t)\ge 0$ over time and conduct a multiplicative drift analysis. To this 
end, we express $g_t$ as a sum of two (approximately) linear functions and distinguish between different cases and subcases. 
In the notation of Lemma~\ref{lem:charac-fitness-quartersphere}, we consider 
the case $\db{t}<0$ first, which means that the whole positive arc of \quartersphere 
has room to lie completely above the line.
Recalling the lemma,  
we have the following characterization 
of the fitness distance:
$g_t = \eta_t + \max\{0,\dphi{t }-\eta_t/2\}$.

We want to show the multiplicative drift $\expect{g_t-g_{t+1}\mid g_t} \ge \delta g_t$ 
for some $\delta=\Theta(1/\ln^2 r)$. Here 
we distinguish between
two subcases. If $\max\{0,\dphi{t}-\eta_t/2\} > \eta_t/2$ (\ie, $\dphi{t} > \eta_t$), 
then, as explained in Lemma~\ref{lem:charac-fitness-quartersphere}, the two endpoints of the positive arc are on different sides of the line
and we analyze steps keeping the bias but moving the angle closer to $\pi/4$. 
Any change of angle 
by an amount $\xi \in \Delta^*\coloneqq [0, \dphi{t}-\eta_t/2]$ is accepted (assuming that the bias is not mutated in the same step) 
and decreases $g_t$ by~$\xi$. Changes of larger amounts 
(up to $\dphi{t}$) are also accepted, as we shall exploit 
in a special case below, but do not lead to an additional 
improvement of the second component of~$g_t$.

In our discretized representation, the changes of angle 
in the set $\Delta^*$  correspond to steps of size~$j$ 
of the angle component in the search point for 
$j=1,\dots,\lfloor (r/(2\pi)) (\dphi{t}-\eta_t/2)\rfloor$. 
If $\dphi{t}-\eta_t/2<2\pi/r$, the previous floor function 
is~$0$, but then already a step 
of size~$1$ in the search space, changing angle in 
the right direction and not changing bias, reduces the 
$\dphi{t}$-value sufficiently 
to have $\dphi{t+1}<\eta_{t+1}/2$ and therefore~$0$ contribution 
of the second component of~$g_t$. Otherwise, \ie, if 
$\dphi{t}-\eta_t/2\ge r/(2\pi)$, then 
$\lfloor (r/(2\pi)) (\dphi{t}-\eta_t/2)\rfloor \ge 
(r/(4\pi)) (\dphi{t}-\eta_t/2) $ 
and the expected progress in the potential space through 
the steps of size~$j$ is at least 
\[
\sum_{j=1}^{\lfloor (r/(2\pi)) (\dphi{t }-\eta_t/2)\rfloor} \frac{1}{jH_r}\cdot j \ge \frac{(r/(4\pi))(\dphi{t}-\eta_t/2) }{H_r},
\]
so along with the probability of not mutating bias, which is $1-1/(2N)=1/2$,  we have
\[
\expect{g_t-g_{t+1}\mid g_t} \ge  \frac{ (\dphi{t}-\eta_t/2)}{8H_r} \ge \frac{g_t}{24H_r}
\]
using that $g_t \le 3 (\dphi{t}-\eta_t/2)$ by assumption 
(following from the representation of~$g_t$ and
the condition for the present case). 
Hence, we even have $\delta\ge 1/(24\ln r+24)$ in this case.

In the other subcase, \ie, if $\max\{0,\dphi{t}-\eta_t/2\} \le \eta_t/2$, we consider fitness improvements 
made by bringing the angle closer 
to its optimum $\pi/4$ and an increase of bias that becomes possible as a consequence of the new angle. More 
precisely, the step should result in $-\db{t+1}\in [\abs{\db{t}/2},\abs{\db{t}}]$ and $\dphi{t+1} \le d' - \pi/4$, where $d'\coloneqq \arccos(b^* + \db{t}/2)$ ,
which means that the endpoints of the positive arc cannot lie below the line after the considered improvement of  bias. We analyze the 
probability of the desired change of angle first. Note that any $(r/2\pi) \phi_{t+1} \in [\pi/2-d',d']$ fulfills the desired change. We bound the length 
of the interval and obtain, using $\arccos(b^*)=\pi/4$ and a Taylor approximation 
for the $\arccos$, that 
$d'-(\pi/2-d') = 2d'-\pi/2=2\arccos(b^* + \db{t}/2) -\pi/2\ge  c \card{\db{t}}$ 
for a constant~$c>0$. Hence,  
the target interval for the angle has length at least~$ (r/(2\pi )) c\card{\db{t}}$ in the search point representation. Without loss of generality, this bound is an integer by choosing~$c$ appropriately. Moreover, since we assume for the current angle that 
$\dphi{t} < \eta_t = 2(\arccos(b^* + \db{t}) - \pi/4) \le 
c' \abs{\db{t}}$ for another constant~$c'>c$, the maximum change of the angle over the target interval is bounded from above by $ (r/(2\pi))c' \abs{\db{t}}$, again 
in the search point representation. 
Again, the bound may be  assumed as an integer.
Together, the probability of changing the angle as desired is at least 
\[
\sum_{j= (r/(2\pi)) (c'-c) \abs{\db{t}}}^{  (r/(2\pi))c' \abs{\db{t}}} \frac{1}{jH_r} = \Omega(1/H_r)
\]
using Lemma~\ref{lem:harmonic-interval}, which even holds if 
$\db{t}$ depends on~$r$. Assuming the desired change of angle, the 
expected decrease in bias is at least
\[
\sum_{j=1}^{\lfloor (r/2)(\db{t}/2)\rfloor} \frac{1}{jH_r} \cdot \frac{2 j}{r} \ge \frac{\card{\db{t}}}{2H_r},
\]
using similar arguments as above to 
analyze the rounding effects. (If 
$\lfloor (r/2)(\db{t}/2)\rfloor = 0$, then a step of size~$1$ in the
search space suffices.)

Combining with the probability 
of changing the angle as desired, the unconditional drift is $\Omega(\card{\db{t}} /\ln^2 r)$. 
Finally, we note that $\card{\db{t}} = b^* - \cos(\arccos(2b_t/r-1)) 
= b^* - \cos(\eta_t/2+\pi/4)$, so using $\cos(\pi/4)=b^*$ and a Taylor 
expansion, we have 
$\card{\db{t}}=\Omega(\eta_t)$. Altogether, since $g_t\le (3/2)\eta_t$ in 
the present subcase, we have
\[
\expect{g_t-g_{t+1}\mid g_t} = \Omega(g_t/\ln^2 r).
\]

We still have to deal with the case $\db{t}>0$, \ie, the bias 
is greater than its optimum value so that at least one end point 
of the positive arc of \quartersphere is below the line. 
Accordingly, the 
formula for $g_t$ derived from Lemma~\ref{lem:charac-fitness-quartersphere} 
reads
\begin{equation*}
    g_t = -\eta_t + \max\{0,\dphi{t}+\eta_t/2\}.
\end{equation*}
The analysis proceeds as before, except for some flipped signs. 
More precisely, 
the subcase that $\dphi{t}+\eta_t/2 > \abs{\eta_t}/2$ is handled 
in the same way and we obtain a drift of $\Omega(g_t/\ln^2 r)$. 
In the complementary subcase $\dphi{t}+\eta_t/2 \le \abs{\eta_t}/2 $, we consider
mutations of angle that allow a decrease of bias to 
a value in $[b^* + \db{t}/2,b^* + \db{t}]$. Also this  probability is 
bounded in the same way as above. Finally, such a decrease  of bias changes the absolute  $\eta_t$-value 
as described above, and 
we obtain the same asymptotic drift bound for the $g_t$-value.

Altogether, 
having established a multiplicative drift for~$g_t$ with $\delta=\Omega(1/\ln^2 r)$ 
in all cases, the bound $O(\log^3 r)$
on the optimization time follow from the 
multiplicative drift theorem and 
$x_{\min}=1/r$, 
noting that the smallest possible fitness distance is~$\Theta(1/r)$.

We still have to consider  an arbitrary 
initialization. As soon as the fitness is strictly larger 
than~$3/4$, we are in the setting from above and 
have an expected optimization time  
$O(\log^3 r)$. Hence, it suffices to analyze the 
expected time until reaching fitness larger than~$3/4$. 
For this it is sufficient to mutate the current 
bias to a value in 
$(0,0.5]$ 
and 
an the angle to a value 
strictly in between  $a\coloneqq \arcsin(0.5)$ and $\pi/2-a$, reusing the analyses 
for the beneficial initialization from the proof 
of the 2nd property. Now,
the target intervals for bias and angle both have 
a length of~$\Omega(r)$ with respect to our 
search space representation. Hence, again using 
Lemma~\ref{lem:harmonic-interval},  such a mutation 
has probability $\Omega(1/\log^2 r)$ and the 
expected time to reach fitness larger than~$3/4$ is 
therefore~$O(\log^2 r)$, which is a lower-order term.
\end{proofof}

We now turn to the case of more than 
one neuron, which is necessary 
to achieve the best possible fitness~$1$ on the problem \twospherequarters. 
As mentioned above, with only~$1$ 
neuron, the fitness cannot exceed~$3/4$.

\begin{theorem}
\label{theo:twospherequarters}
Let $r=8k$ for an integer~$k$. 
With at least constant probability, 
the optimization time 
of the standard \oona and the local \oona with $N=2$  on the problem 
\twospherequarters  is $O(\log^3 r)$ and $O(r)$, respectively. 

For~$N=1$, the same bounds on the optimization time 
as in as Theorem~\ref{theo:quartersphere} apply.
\end{theorem}

\begin{proof}The second paragraph follows in the same way as Theorem~\ref{theo:quartersphere}. The only difference 
is that there are two global optima for the location of the line. Again, if fitness improvements are found by increasing the bias and moving the line towards 
the boundary of the unit circle, then 
we analyze the time 
to arrive at a beneficial angle that allows 
fitness improvements by lowering the bias again.

We now turn to the first paragraph of the theorem, 
again using similar arguments as in the proof 
of Theorem~\ref{theo:quartersphere}. The difference is now 
that there are two lines that 
the algorithm can move and rotate. The global optimum is taken 
when the angles are $\pi/4$ and $5\pi/4$ and both biases are $\sqrt{2}/2$. We will consider an initialization where the two lines 
are initialized in the two ``basins of attraction'' belonging to this optimal setting. Moreover, the distance of the lines 
(measured within the circle)  
will be at least constant to allow 
an application of Lemma~\ref{lem:charac-local-opt}.

Let $(\phi_1,b_1)$ and $(\phi_2,b_2)$ be the initial 
angles and biases of the two neurons. We 
consider the joint event that $2rb_1-1\in [0.6,0.65]$, $2rb_2-1\in [-0.6,-0.65]$, $2\pi r\phi_1\in[\pi/4-\alpha^*,\pi/4+\alpha^*]$ and 
$2\pi r\phi_2\in[5\pi/4-\alpha^*, 5\pi/4+\alpha^*]$ for $\alpha^*=\arccos(0.6)-\pi/4=0.141897\dots$. This event happens 
with constant probability. Then, by simple trigonometry, the 
positive arc of \twospherequarters in the first quadrant 
lies above the line of the first 
neuron and the arc
in the third quadrant above the line of the 
second neuron. Moreover
the choice of bias leaves strictly less than a quarter of the points above 
either line  
wrongly classified. Altogether, this initialization gives a 
total fitness of more than $3/4$. 
This has the following implications on the future 
placement of the two lines. To 
achieve at least the same 
fitness with a different placement of 
the lines, 
at least a part of each of the two positive arcs of \quartersphere 
must lie above 
a line, and each line must have a positive part above it. Otherwise, a
positive arc of length~$\pi/2$ would 
lie below both lines and the fitness could 
not be greater than~$3/4$.
Hence, one line cannot move completely above the other. Moreover, with 
the assumed fitness, it is impossible 
to reach placements such that the lines  intersect each other within the unit circle. If such an intersection happened, 
since both lines have positive arc pieces 
above them, the angle between the lines, taken in the area above them, would be 
at least~$\pi$. Then 
also 
a negative part of the unit circle 
of length at least~$\pi/2$ would lie 
above a line, contradicting the fitness
strictly larger than~$3/4$.

We now complete the proof re-using the analyses from 
Theorem~\ref{theo:quartersphere} under 
the beneficial initialization. We consider the local  \oona first. 
Then  
there is for each hyperplane an event or a 
sequence of constant many~events of constant 
probability that brings the positive 
arc  in the respective quadrant (1st or 3rd) closer to the 
line, \ie, decreases the total 
length of wrongly classified arcs above the line without making the
length of wrongly classified points below the hyperplane bigger. This holds 
until each bias deviates 
by at less than~$2/r$ from its optimum $\sqrt{2}/2$. The expected time 
to reach this state (still assuming 
the beneficial initialization) is $O(r)$, and along with Markov's inequality 
the joint probability 
of finding the global optimum in $O(r)$ steps is at least constant.

For the standard \oona with 
harmonic mutation,  express the fitness
distance as a generalization of the expression in Lemma~\ref{lem:charac-fitness-quartersphere}, where 
we consider the quantities $\db{t}$ and $\dphi{t}$ separately for the 
two lines and add up the wrongly classified parts 
in the style of $\eta_t + \max\{0,\dphi{t}-\eta_t/2\}$ for both lines, 
noting that the lines do not intersect each other.
Then we conduct the analysis from the proof of 
Theorem~\ref{theo:quartersphere} (conditioning on that 
a step only changes the parameter of one neuron) 
and obtain an expected time of $O(\log^3 r)$ 
until the fitness distance has reached its minimum ($\pm 1$ in the 
search point representation). 
We arrive at  the claimed bound 
$O(\log^3 r)$. 
\end{proof}

The constant success probability implies 
that multi-start variants of the 
algorithm are highly efficient. See, \eg, 
\cite{WegenerICALP05} for definitions 
and analyses of multi-start schemes. 
However, without restarts it 
is not clear whether there is a general finite bound for 
the local \oona and a bound for the 
harmonic \oona that is better than $O((r \log r)^c)$ for a constant~$c<1$ (see also the 
worst case bound from Lemma~\ref{lem:general-upper}). 
 The problem is that the two lines 
of the neurons may intersect 
in such a way that exactly one positive quadrant 
of \twospherequarters is classified positively. 
This situation essentially ``tilts'' and locks the lines from moving, 
except for a random walk of the intersecting point. See the left-hand side of 
Figure~\ref{fig:quartersphere-tilt} for an illustration. 
Also, one line could be lying completely above the other one and therefore be irrelevant for fitness
evaluation. This irrelevant line can freely perform 
a random walk about the relevant one, \ie, 
in contrast to the analysis of
Theorem~\ref{theo:quartersphere} (part~1), the random walk would not be limited to 
configurations with bias~$1$. See the right-hand side of Figure~\ref{fig:quartersphere-tilt}.

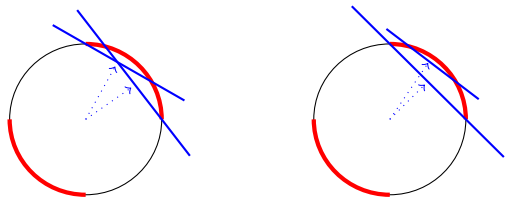
\begin{figure}[ht]
\centerline{\begin{tikzpicture}
\draw (0,0) circle (1);
\draw[red, ultra thick] (1,0) arc (0:90:1);
\draw[red, ultra thick] (-1,0) arc (180:270:1);
\node (A) at (90:1) {};
\node (B) at (30:1) {};
\node (C) at (75:1) {};
\node (D) at (0:1) {};
\draw [add= -1.5 and -1.5, blue, thick] (A) to (B);
\draw [add= -1.5 and -1.5, blue, thick] (C) to (D);
\draw[blue,->,dotted] (0,0) -> (35:0.73);
\draw[blue,->,dotted] (0,0) -> (61:0.8);
\begin{scope}[shift={(4,0)}]
\draw (0,0) circle (1);
\draw[red, ultra thick] (1,0) arc (0:90:1);
\draw[red, ultra thick] (-1,0) arc (180:270:1);
\node (A) at (90:1) {};
\node (B) at (0:1) {};
\node (C) at (75:1) {};
\node (D) at (30:1) {};
\draw [add= -1.5 and -1.5, blue, thick] (A) to (B);
\draw [add= -1.5 and -1.5, blue, thick] (C) to (D);
\draw[blue,->,dotted] (0,0) -> (45:0.65);
\draw[blue,->,dotted] (0,0) -> (56:0.9);
\end{scope}
\end{tikzpicture}}
\caption{Left: two lines 
are tilted and impede further progress. Right: one line lies completely above the other and does not contribute to fitness.}
\label{fig:quartersphere-tilt}
\end{figure}

As a next step, one could  consider generalized definitions  
of $\twospherequarters$ to problems 
with $k$ equally spaced positive 
arcs of identical size. It is not difficult to see that these would 
require $N=k$ neurons to achieve fitness~$1$. Moreover, a generalization of Theorem~\ref{theo:twospherequarters}
seems plausible, with a probability of at least~$c^{-k}$ for a constant~$c$ for
optimizing 
the problem efficiently. In any case, it would be interesting to consider 
this problem in an analysis where both $r$ and~$N$ 
are asymptotically growing.

Finally, we turn to the 
 problem \spherelocalopt, which possesses a local optimum that 
 is hard to escape and causes infinite expected 
 optimization time for the local \oona. However, for the standard \oona with 
 harmonic mutation, it is not difficult.

\begin{theorem}
\label{theo:spherelocalopt}
With at least constant probability, 
the local \oona with $N=1$ on the problem 
\spherelocalopt reaches a local 
optimum of fitness  at most $2/3$ and 
cannot make improvements from there.

 However, the expected optimization time of the \oona with 
 harmonic mutation is $O(\log^3 r)$.
\end{theorem}

\begin{proof}We consider an initialization where 
\begin{itemize}
\item the 
line belonging to the neuron 
intersects the two negative
 arcs 
of the unit circle with polar 
angle in the 
interval $[\pi/3,2\pi/3]$ and 
$[11\pi/6,2\pi]$ and 
\item 
the positive area of \spherelocalopt corresponding to angles 
$[0,\pi/3]$ is completely below the line.
\end{itemize}
This corresponds to a location 
shown in the right-hand side of 
Figure~\ref{threarc:fig} up to shifts of the hyperplane within the negative area 
mentioned above. 
The fitness of such an initial search point 
is between $5/12$ and $2/3$. A sufficient condition for this 
initialization is a bias within $[-\cos(18\pi/48),-\cos(19\pi/48)]$ 
and an angle within $[7/12\pi-a,7/12\pi+a]$, where $a=\pi/48$. 
Since
 the considered initialization specifies  
 intervals of length~$\Omega(r)$ for 
 both angle and radius, 
 its  probability is at least a constant. 

To improve the 
fitness to above~$2/3$ it is necessary to change 
the angle of the line in one step 
by at least $\pi/6$, which is impossible with 
the $\pm 1$ mutation of the \oona.

For the harmonic mutation, we analyze the event of a mutation 
leading to fitness strictly larger than~$2/3$. For this it is sufficient 
to mutate the bias to an interval of sufficiently small but constant size around~$0$ and the angle 
to an interval of sufficiently small constant size around~$11\pi/6$. By 
Lemma~\ref{lem:harmonic-interval},
the expected time 
for this to happen is $O(\log^2 r)$. Afterwards, we can express the 
fitness distance in a way analogous to the proof of the first statement of Theorem~\ref{theo:quartersphere}), 
except for that there is always a negative arc of length $\pi/6$ above the line 
that is classified wrongly. The expected time to reach optimal 
bias and angle (up to the error $\pm 1$ in the search point representation) 
is $O(\log^3 r)$ by the same multiplicative drift analysis as in
the proof of 
Theorem~\ref{theo:quartersphere}.
\end{proof}

Theorem~\ref{theo:spherelocalopt} does 
not exclude that multi-start variants 
of the local \oona are efficient; in fact 
we think that with probability~$\Omega(1)$, 
it finds the global optimum of fitness $3/4$ in polynomial time.

\section{Experiments}
\label{sec:experiments}

We ran the \oona with local and hybrid 
mutation on the problems \halfsphere, \quartersphere and \spherelocalopt with 
$N=1$ and on \twospherequarters with 
$N=2$, each 100 times. 
We canceled the runs after stagnation phases without fitness improvement 
of length~$100r\log r$, where the latter choice was inspired by the lower bound $1/(rH_r)$ of the harmonic mutation operator hitting an
arbitrary state. Finally, for 
\quartersphere with $N=2$ we considered 
not only the ANN from Figure~\ref{fig:twolayer} 
with hard-wired OR in the output layer (which was the network assumed in  
Theorem~\ref{theo:twospherequarters}), 
but  also tried a variant of the 
ANN  where 
the parameters of the final neuron were 
evolved in the same way as the first two neurons. 

Experiments for $D=2$ and 
$r$ growing 
between $120$ and $1200$ supplement the
theoretical running time bounds from 
Section~\ref{sec:runtime-analysis} nicely. 
In particular, the high standard deviation 
for \twospherequarters show that the 
polylogarithmic running time in case of 
beneficial initialization is not always 
the case and that the pathological configurations shown in Figure~\ref{fig:quartersphere-tilt} are hard
 to overcome. However, with the timeouts specified above the hybrid mutation was generally efficient on all problems, while the local mutation struggled especially on \spherelocalopt and 
 \twospherequarters. Detailed tables of 
 all experimental data are 
 found in Tables~\ref{tab:halfnonelocal}--\ref{table:exp-cma} of Appendix~\ref{sec:experiment-tables}.

We also ran  CMA-ES, a state-of-the-art evolutionary algorithm for continuous 
spaces \cite{HansenOstermeierECJ01}, on our benchmark problems, still fixing $D=2$. To achieve a fair comparison with the \oona, we use essentially the same 
representation 
with polar angle and bias on the intervals 
$[0,2\pi]$ and $[-1,1]$, respectively, enforced in CMA-ES by 
bounding the components of its real-valued vectors.
Experiments 
were performed using the \texttt{cma~3.3.0} Python package, initialized 
with a standard deviation of~$1$, expected starting solution in 
the origin and default 
stopping criterion. Table~\ref{table:exp-cma} 
shows the results we recorded. On \halfsphere 
and \quartersphere, CMA-ES performed generally well and found the optimum almost always within  $2\%$ of optimality and even more frequently 
within $5\%$; in very rare cases on \quartersphere,
the algorithm was stuck more than $50\,\%$ 
from the optimum. The picture was similar on 
\spherelocalopt, which was solved to optimality (up to~$2\%$ tolerance) in 
almost all cases. The unsuccessful runs were usually 
stuck at fitness about~$2/3$, corresponding 
to a local optimum.

On \twospherequarters with a hard-wired 
OR in the output layer (as shown 
in Figure~\ref{fig:twolayer}), CMA-ES frequently 
missed the optimum and was stuck at fitness~$3/4$; the success rate was only about 
$33.3\,\%$. Again, we also tried the  variant of the 
ANN where 
the parameters of the final neuron were 
evolved  as well. 
Here only a success rate  of $1.8\,\%$ was observed.

The experiments on
\twospherequarters where the \oona with 
hybrid mutation had a higher 
success rate than CMA-ES  may be biased by the default 
stopping criterion in CMA-ES. A 
longer stagnation phase may enable it to 
find the global optimum more frequently. 
We also 
experimented with a Cartesian representation for 
the neurons' weights 
in the CMA-ES, but did not observe 
improvements.

Additional experiments showed that the approach
taken in this 
paper also is successful in higher-dimensional settings.
The hyperplanes (neurons) are still represented by an normal vector in spherical coordinates (angles) and
a bias. The angles and the bias are changed with fixed step width. Two-layer
networks were used, where the single neuron in the output layer computed a Boolean function.
The test were run on point sets. A simple example is a set of eight points
located at the corners of a cube. Four non-neighboring 
corners are being labeled ``1'' and the others ``0''. 
In most cases, the network was able to reach a perfect classification efficiently.

We also have experimentally investigated using a continuous, heavy-tailed distribution
to mutate angle and bias instead of the discrete harmonic distribution.
The experiments do not show a qualitative difference, \ie, optimal solutions
are found with similar frequency. However, the number of steps to reach the
optimum varies depending on the chosen distribution (shifted Pareto, exponential, Cauchy)
and the setting of the parameters for the distribution. We found parameter settings
which gave a performance similar to the one when using the discrete, harmonic distribution but 
did not achieve a significant performance gain over the harmonic distribution. 
This clearly is a field for further investigation.

\section{Conclusions}
We have proposed an algorithmic framework for the 
runtime analysis of problems in neuroevolution. 
The framework comprises a simple evolutionary algorithm 
called \oona for the 
optimization of parameters of neurons, more precisely
weights and biases, and a scalable network structure 
with two layers (hidden and output) as search space for 
optimization problems. We also have proposed simple 
benchmarks based on labeled points on the unit hyperspere
and used them to illustrate typical behavior and 
challenges for the search trajectory of the \oona. 
We have identified problems with local optima and 
compared two types of mutation operators, where 
the so-called harmonic mutation often gives 
exponentially better runtime bounds. Experimental 
supplements show that the proven runtime bounds and 
performance difference are pronounced in practice 
already for small problem sizes.

In this first study of the runtime 
of neuroevolutionary algorithms, 
we have only scratched 
the surface of the rich structure 
arising already from very simple problems.  
So far we are working with fixed structures 
for the artificial neural networks, while 
state-of-the-art neuroevolutionary algorithms 
would also evolve the networks' topology. 
Moreover, the present runtime analyses 
are limited to the case of $2$ dimensions, 
while the general problem definitions call 
for an analysis in higher dimensions. 
Furthermore, more advanced classification 
problems could be considered. 
We see also room for improvement in the search operators. For example, 
 in certain cases, self-adaptation of 
the mutation strength may lead 
to a runtime of $O(\log(1/\epsilon))$ to achieve 
an $\epsilon$-approximation of the optimum. We 
leave all these considerations as subjects 
for future research.

\subsection*{Acknowledgement} This work was supported by a grant 
from the Independent Research Fund Denmark (grant no.~2032-00101B).

\bibliographystyle{alpha}
\bibliography{references}

\begin{thebibliography}{DRWW11}

\bibitem[BR88]{BlumRivestNIPS88}
Avrim Blum and Ronald~L. Rivest.
\newblock Training a 3-node neural network is np-complete.
\newblock In {\em Proc.\ of NIPS~1988}, pages 494--501. Morgan Kaufmann, 1988.

\bibitem[DDK18]{DoerrDKAlgorithmica18}
Benjamin Doerr, Carola Doerr, and Timo K{\"{o}}tzing.
\newblock Static and self-adjusting mutation strengths for multi-valued
  decision variables.
\newblock {\em Algorithmica}, 80(5):1732--1768, 2018.

\bibitem[DDL21]{DoerrDLAlgo21}
Benjamin Doerr, Carola Doerr, and Johannes Lengler.
\newblock Self-adjusting mutation rates with provably optimal success rules.
\newblock {\em Algorithmica}, 83(10):3108--3147, 2021.

\bibitem[DJW02]{DJW02}
Stefan Droste, Thomas Jansen, and Ingo Wegener.
\newblock On the analysis of the {(1+1)} evolutionary algorithm.
\newblock {\em Theoretical Computer Science}, 276(1-2):51--81, 2002.

\bibitem[DJW12]{DoerrJWAlgo12}
Benjamin Doerr, Daniel Johannsen, and Carola Winzen.
\newblock Multiplicative drift analysis.
\newblock {\em Algorithmica}, 64(4):673--697, 2012.

\bibitem[DN20]{DoerrNeumann20}
Benjamin Doerr and Frank Neumann, editors.
\newblock {\em Theory of Evolutionary Computation -- Recent Developments in
  Discrete Optimization}.
\newblock Natural Computing Series. Springer, 2020.

\bibitem[DN21]{DNTELOSurvey21}
Benjamin Doerr and Frank Neumann.
\newblock A survey on recent progress in the theory of evolutionary algorithms
  for discrete optimization.
\newblock {\em ACM Transations on Evolutionary Learning and Optimization},
  1(4):1–43, 2021.

\bibitem[DP12]{DoerrPohlGECCO12}
Benjamin Doerr and Sebastian Pohl.
\newblock Run-time analysis of the {(1+1)} evolutionary algorithm optimizing
  linear functions over a finite alphabet.
\newblock In {\em Proc.\ of GECCO~2012}, pages 1317--1324. {ACM Press}, 2012.

\bibitem[DRWW11]{DietzfelbingerRWWAlgorithmica11}
Martin Dietzfelbinger, Jonathan~E. Rowe, Ingo Wegener, and Philipp Woelfel.
\newblock Precision, local search and unimodal functions.
\newblock {\em Algorithmica}, 59(3):301--322, 2011.

\bibitem[GM21]{GalvanMooneyNeuroevolutionSurvey}
Edgar Galván and Peter Mooney.
\newblock Neuroevolution in deep neural networks: Current trends and future
  challenges.
\newblock {\em IEEE Transactions on Artificial Intelligence}, 2(6):476--493,
  2021.

\bibitem[Gun05]{GuniaGECCO05}
Christian Gunia.
\newblock On the analysis of the approximation capability of simple
  evolutionary algorithms for scheduling problems.
\newblock In {\em Proc.\ of GECCO~2005}, pages 571--578. {ACM Press}, 2005.

\bibitem[HO01]{HansenOstermeierECJ01}
Nikolaus Hansen and Andreas Ostermeier.
\newblock Completely derandomized self-adaptation in evolution strategies.
\newblock {\em Evolutionary Computation}, 9(2):159--195, 2001.

\bibitem[HOS20]{HallOSPPSN20}
George~T. Hall, Pietro~S. Oliveto, and Dirk Sudholt.
\newblock Fast perturbative algorithm configurators.
\newblock In {\em Proc.\ of PPSN~2020}, volume 12269, pages 19--32. Springer,
  2020.

\bibitem[J{\"{a}}g07]{JaegerskuepperECJ07}
Jens J{\"{a}}gersk{\"{u}}pper.
\newblock Algorithmic analysis of a basic evolutionary algorithm for continuous
  optimization.
\newblock {\em Theoretical Computer Science}, 379(3):329--347, 2007.

\bibitem[Jan13]{Jansen13}
Thomas Jansen.
\newblock {\em Analyzing Evolutionary Algorithms -- The Computer Science
  Perspective}.
\newblock Springer, 2013.

\bibitem[KK19]{KotzingKrejcaTCS19}
Timo K{\"{o}}tzing and Martin~S. Krejca.
\newblock First-hitting times under drift.
\newblock {\em Theoretical Computer Science}, 796:51--69, 2019.

\bibitem[KLW15]{KotzingLWFOGA15}
Timo K{\"{o}}tzing, Andrei Lissovoi, and Carsten Witt.
\newblock {(1+1)} {EA} on generalized dynamic onemax.
\newblock In {\em Proc.\ of FOGA~2015}, pages 40--51. {ACM Press}, 2015.

\bibitem[MD89]{MontanaD89}
David~J. Montana and Lawrence Davis.
\newblock Training feedforward neural networks using genetic algorithms.
\newblock In {\em Proc.\ of IJCAI~1989}, pages 762--767. Morgan Kaufmann, 1989.

\bibitem[NW10]{NeumannW10}
Frank Neumann and Carsten Witt.
\newblock {\em Bioinspired Computation in Combinatorial Optimization --
  Algorithms and Their Computational Complexity}.
\newblock Springer, 2010.

\bibitem[PD17]{PintoDoerrEA2017}
Eduardo~Carvalho Pinto and Carola Doerr.
\newblock Discussion of a more practice-aware run-time analysis for
  evolutionary algorithms.
\newblock In {\em Proc. of Artificial Evolution (EA~2017)}, pages 298--305,
  2017.

\bibitem[Rot06]{RothlaufRepresentationsBook}
Franz Rothlauf.
\newblock {\em Representations for Genetic and Evolutionary Algorithms}.
\newblock Springer, 2nd edition, 2006.

\bibitem[SB18]{SuttonBartoRL2018}
Richard~S. Sutton and Andrew~G. Barto.
\newblock {\em Reinforcement Learning: An Introduction}.
\newblock MIT Press, Cambridge, MA, 2nd edition, 2018.

\bibitem[SCLM19]{StanleyCLM19}
Kenneth~O. Stanley, Jeff Clune, Joel Lehman, and Risto Miikkulainen.
\newblock Designing neural networks through neuroevolution.
\newblock {\em Nature Machine Intelligence}, 1(1):24--35, 2019.

\bibitem[SM02]{stanley2002evolving}
Kenneth~O Stanley and Risto Miikkulainen.
\newblock Evolving neural networks through augmenting topologies.
\newblock {\em Evolutionary computation}, 10(2):99--127, 2002.

\bibitem[UB22]{UnalEvolutionaryDesignNN}
Hamit~Taner \"{U}nal and Fatih Ba\c{s}\c{c}ift\c{c}i.
\newblock Evolutionary design of neural network architectures: A review of
  three decades of research.
\newblock {\em Artificial Intelligence Review}, 55(3):1723–1802, 2022.

\bibitem[\v{S}02]{SimaTrainingSigmoidalHard2002}
Ji\v{r}\'{i} \v{S}\'{i}ma.
\newblock {Training a Single Sigmoidal Neuron Is Hard}.
\newblock {\em Neural Computation}, 14(11):2709--2728, 2002.

\bibitem[Weg01]{WegenerICALP01}
Ingo Wegener.
\newblock Theoretical aspects of evolutionary algorithms.
\newblock In {\em Proc.\ of ICALP~2001}, volume 2076, pages 64--78. Springer,
  2001.

\bibitem[Weg05]{WegenerICALP05}
Ingo Wegener.
\newblock Simulated annealing beats metropolis in combinatorial optimization.
\newblock In {\em Proc.\ of ICALP~2005}, volume 3580, pages 589--601. Springer,
  2005.

\bibitem[WH89]{WhitleyHansonICGA89}
Darrell Whitley and Thomas Hanson.
\newblock Optimizing neural networks using faster, more accurate genetic
  search.
\newblock In {\em Proc.\ of ICGA~1989}, page 391–396. Morgan Kaufmann, 1989.

\end{thebibliography}

\newpage

\begin{appendix}

\section{Detailed Experimental Data}
\label{sec:experiment-tables}

This section provides tables detailing the 
statistics of the experiments we performed 
with the \oona and CMA-ES on the 
problems \halfsphere, \quartersphere, \twospherequarters 
and \spherelocalopt. The tables for the \oona 
specify the resolution~$r$, the percentage 
of runs reaching the global optimum (up to 
the usual tolerance) before the run was forced 
terminated as well as mean, standard deviation and median of the number of $f$-evaluations. Runs were terminated after a stagnation phase 
specified in the captions. For problem \twospherequarters we also performed experiments 
where the hard-wired OR in the output layer is
replaced by a modifiable neuron of the same type as the ones in 
the hidden layer. The rate of correct classifications is very high,
while the runtime is about three to four times as large as the one for 
the situation with a hard-wired OR. For all simulations, besides the CMA-ES,
100 runs have been performed.

The table for CMA-ES is structured similarly, 
but does not have the resolution parameter and 
lists all problems at once. We also display 
the average final fitness.

\begin{table}[h]
  \centering
  \begin{tabular}{|r|r|r|r|r|}
  \hline
  $r$ & \%opt & mean & sdev& median\\\hline 
120 & 99,0 & 210 & 404 & 16\\ \hline
240 & 100,0 & 829 & 3317 & 91\\ \hline
360 & 100,0 & 3041 & 14035 & 152\\ \hline
480 & 99,0 & 1952 & 8158 & 538\\ \hline
600 & 100,0 & 1525 & 4325 & 576\\ \hline
720 & 100,0 & 4687 & 18135 & 154\\ \hline
840 & 100,0 & 3860 & 15002 & 871\\ \hline
960 & 98,0 & 3824 & 17370 & 603\\ \hline
1080 & 100,0 & 7862 & 32306 & 412\\ \hline
1200 & 97,0 & 7155 & 27766 & 421\\ \hline
\hline
\end{tabular}
\caption{Problem \halfsphere: \oona with 
$N=1$, local mutation,   maximal number of steps $100 r \log(r)$.}
\label{tab:halfnonelocal}

\end{table}

\begin{table}[h]
  \centering
  \begin{tabular}{|r|r|r|r|r|}
  \hline
  $r$ & \%opt & mean & sdev& median\\\hline 
120 & 100,0 & 75 & 107 & 19\\ \hline
240 & 100,0 & 121 & 162 & 26\\ \hline
360 & 100,0 & 175 & 222 & 64\\ \hline
480 & 100,0 & 266 & 351 & 55\\ \hline
600 & 100,0 & 247 & 364 & 55\\ \hline
720 & 100,0 & 335 & 482 & 29\\ \hline
840 & 100,0 & 400 & 541 & 86\\ \hline
960 & 100,0 & 403 & 574 & 97\\ \hline
1080 & 100,0 & 576 & 918 & 90\\ \hline
1200 & 100,0 & 612 & 896 & 89\\ \hline
\hline
\end{tabular}
\caption{Problem \halfsphere: \oona with $N=1$, 
harmonic mutation,   maximal number of steps  $100r\log(r)$.}
\label{half_harm:tab}

\end{table}

\begin{table}[h]
  \centering
  \begin{tabular}{|r|r|r|r|r|}
  \hline
  $r$ & \%opt & mean & sdev& median\\\hline 
  120 & 99,0 & 4287 & 9119 & 632\\ \hline
240 & 94,0 & 16508 & 34386 & 1432\\ \hline
360 & 97,0 & 26581 & 51291 & 2379\\ \hline
480 & 95,0 & 42417 & 80192 & 2942\\ \hline
600 & 97,0 & 51452 & 100703 & 3493\\ \hline
720 & 93,0 & 65050 & 139080 & 3800\\ \hline
840 & 95,0 & 68656 & 152475 & 4930\\ \hline
960 & 89,0 & 114155 & 217364 & 6038\\ \hline
1080 & 90,0 & 139336 & 254946 & 5964\\ \hline
1200 & 97,0 & 104877 & 223079 & 6623\\ \hline
\hline
\end{tabular}
\caption{Problem \quartersphere: \oona with $N=1$, local mutation, max no. of steps  $100 r \log(r)$.}
\label{BLAc:fig}

\end{table}

\begin{table}[h]
  \centering
  \begin{tabular}{|r|r|r|r|r|}
  \hline
  $r$ & \%opt & mean & sdev& median\\\hline 
120 & 100,0 & 366 & 343 & 249\\ \hline
240 & 100,0 & 675 & 628 & 469\\ \hline
360 & 100,0 & 1168 & 908 & 1052\\ \hline
480 & 100,0 & 1663 & 1509 & 1262\\ \hline
600 & 100,0 & 1819 & 1673 & 1323\\ \hline
720 & 100,0 & 2031 & 2128 & 1288\\ \hline
840 & 100,0 & 2578 & 2008 & 2121\\ \hline
960 & 100,0 & 2519 & 2237 & 2096\\ \hline
1080 & 100,0 & 3228 & 2793 & 2628\\ \hline
1200 & 100,0 & 3243 & 3087 & 2285\\ \hline
\hline
\end{tabular}
\caption{Problem \quartersphere:  
\oona with $N=1$, harmonic mutation, max no. of steps  $100 r \log(r)$.}
\label{BLAd:fig}

\end{table}

\begin{table}[h]
  \centering
  \begin{tabular}{|r|r|r|r|r|}
  \hline
  $r$ & \%opt & mean & sdev& median\\\hline 
  120 & 18,0 & 49279 & 18348 & 57449\\ \hline
240 & 11,0 & 120276 & 34723 & 131535\\ \hline
360 & 10,0 & 194072 & 56347 & 211899\\ \hline
480 & 18,0 & 247779 & 108057 & 296341\\ \hline
600 & 12,0 & 341251 & 119391 & 383815\\ \hline
720 & 12,0 & 418789 & 148949 & 473706\\ \hline
840 & 9,0 & 525559 & 135531 & 565605\\ \hline
960 & 12,0 & 589613 & 196356 & 659225\\ \hline
1080 & 8,0 & 697520 & 195195 & 754349\\ \hline
1200 & 10,0 & 766042 & 254300 & 850809\\ \hline
\hline
\end{tabular}
\caption{Problem \twospherequarters,  
\oona with $N=2$, local mutation, hard-wired OR, max. no. of steps $100 r\log(r)$.}
\label{BLAe:fig}

\end{table}

\begin{table}[h]
  \centering
  \begin{tabular}{|r|r|r|r|r|}
  \hline
  $r$ & \%opt & mean & sdev& median\\\hline 
  120 & 91,0 & 2601 & 4734 & 1080\\ \hline
240 & 100,0 & 7148 & 9269 & 3466\\ \hline
360 & 100,0 & 9159 & 10523 & 6705\\ \hline
480 & 100,0 & 17844 & 18012 & 13904\\ \hline
600 & 100,0 & 26618 & 27355 & 17948\\ \hline
720 & 100,0 & 36407 & 36013 & 26836\\ \hline
840 & 100,0 & 50800 & 49451 & 28861\\ \hline
960 & 100,0 & 54137 & 79541 & 34964\\ \hline
1080 & 100,0 & 70460 & 68565 & 48329\\ \hline
1200 & 100,0 & 75158 & 79424 & 54672\\ \hline
\hline
\end{tabular}
\caption{Problem \twospherequarters,  
\oona with harmonic mutation, $N=2$, hard-wired OR in output layer, max. no. of steps  $100r\log(r)$.
The fact that only 91\% of the runs for $r=120$ reached the optimum is due to
the choice of the tolerance for accepting an solution as optimal.}
\label{twoQ_harm_hard:tab}

\end{table}

\begin{table}[h]
  \centering
  \begin{tabular}{|r|r|r|r|r|}
  \hline
  $r$ & \%opt & mean & sdev& median\\\hline 
  120 & 71,0 & 3247 & 4542 & 2156\\ \hline
240 & 100,0 & 14363 & 15319 & 9584\\ \hline
360 & 100,0 & 38057 & 47235 & 21908\\ \hline
480 & 99,0 & 70533 & 75353 & 42753\\ \hline
600 & 100,0 & 92401 & 92080 & 57884\\ \hline
720 & 100,0 & 126759 & 109118 & 93954\\ \hline
840 & 100,0 & 182077 & 170464 & 125612\\ \hline
960 & 100,0 & 228725 & 197834 & 147276\\ \hline
1080 & 99,0 & 261076 & 232265 & 158192\\ \hline
1200 & 100,0 & 329491 & 278210 & 214112\\ \hline
\hline
\end{tabular}
\caption{Problem \twospherequarters, \oona with harmonic mutation, $N=3$, two neurons in the hidden layer, one in output
layer,  all neurons are exposed to the same kind of mutations (angle and bias),
max. no.  of steps $100 r \log(r)$.}
\label{twoQ_harm_soft:tab}

\end{table}

\begin{table}[h]
  \centering
  \begin{tabular}{|r|r|r|r|r|r|r|}
  \hline
 $r$ & \%opt & mean & sdev& median\\\hline
  120 & 36,0 & 36835 & 27486 & 57449\\ \hline
240 & 45,0 & 72727 & 65040 & 131535\\ \hline
360 & 40,0 & 127395 & 103496 & 211899\\ \hline
480 & 51,0 & 145648 & 147710 & 6056\\ \hline
600 & 44,0 & 215350 & 190054 & 383815\\ \hline
720 & 36,0 & 303936 & 226393 & 473706\\ \hline
840 & 46,0 & 306070 & 281200 & 565605\\ \hline
960 & 37,0 & 415818 & 317617 & 659225\\ \hline
1080 & 41,0 & 445822 & 370107 & 754349\\ \hline
1200 & 47,0 & 454009 & 421556 & 850809\\ \hline
\hline
\end{tabular}
\caption{Problem \spherelocalopt: \oona with local mutation, $N=1$,
 max.\ no.\ of steps $100r\log(r)$.
The table shows how may percent of the runs reached the
optimum of $3/4$.}
\label{BLAu:fig}

\end{table}

\begin{table}[h]
  \centering
  \begin{tabular}{|r|r|r|r|r|}
  \hline
  $r$ & \%opt & mean & sdev& median\\\hline
  120 & 100,0 & 299 & 287 & 226\\ \hline
240 & 100,0 & 512 & 556 & 345\\ \hline
360 & 100,0 & 873 & 1157 & 556\\ \hline
480 & 100,0 & 1180 & 1185 & 754\\ \hline
600 & 100,0 & 1439 & 1994 & 782\\ \hline
720 & 100,0 & 1910 & 2724 & 956\\ \hline
840 & 100,0 & 1911 & 2344 & 1141\\ \hline
960 & 100,0 & 2969 & 4329 & 1217\\ \hline
1080 & 100,0 & 2665 & 3044 & 1374\\ \hline
1200 & 100,0 & 3218 & 4128 & 1755\\ \hline
\hline
\end{tabular}
\caption{Problem \spherelocalopt: \oona with harmonic mutation, $N=1$,
 max.\ no.\ of steps $100r\log(r)$.}
\label{BLAv:fig}
\end{table}

\begin{table}[h]
  \centering
  \begin{tabular}{|r|r|r|r|r|}
  \hline
  Problem & \% opt (2\,\% tolerance) & mean & sdev & avg. fitness\\\hline 
  \halfsphere & 99.6 & 154.8 & 81.4 & 998.0 \\ \hline
\quartersphere & 99.3 & 195.4 & 34.2 & 998.8 \\ \hline
\twospherequarters, hard-wired OR in output layer & 33.3 & 578.3 & 284.1 & 831.2 \\ \hline
\twospherequarters, evolved
output layer & 1.8 & 633.0 & 256.7 & 707.9\\ \hline
\spherelocalopt & 99.6 & 206.8 & 105.1 & 747.5\\ \hline
\hline
\end{tabular}
\caption{CMA-ES with default stopping 
criterion run on 
the different problems; 1000 runs}
\label{table:exp-cma}
\end{table}

\end{appendix}

\end{document}